%% file: arxiv.tex
\documentclass[11pt]{article}
\usepackage[utf8]{inputenc}
\usepackage[margin=1in]{geometry}

\usepackage{microtype}
\usepackage{subfigure}
\usepackage{booktabs}
\usepackage[round]{natbib}

\input{math_commands}

\usepackage{wrapfig,makecell,multirow,booktabs}
\usepackage{algpseudocode,algorithm}
\usepackage{xcolor}
\usepackage[colorlinks=true,citecolor=blue]{hyperref}

\begin{document}

\title{\bf Recovering Linear Causal Models with Latent Variables\\ via  Cholesky Factorization of Covariance Matrix}

\author{\vspace{0.3in}\\
  \textbf{Yunfeng Cai, Xu Li, Minging Sun, Ping Li} \\\\
  Cognitive Computing Lab\\
  Baidu Research\\
No.10 Xibeiwang East Road, Beijing 100193, China\\
  10900 NE 8th St. Bellevue, Washington 98004, USA\\\\
  \texttt{ \{yunfengcai09, dongfangyixi,  sunmingming01, pingli98\}@gmail.com}
}
\date{\vspace{0.2in}}
\maketitle

\begin{abstract}\vspace{0.2in}
\noindent Discovering the causal relationship via recovering the directed acyclic graph (DAG) structure from the observed data is a well-known challenging combinatorial problem. When there are latent variables, the problem becomes even more difficult.
In this paper, we first propose a DAG structure recovering algorithm,
which is based on the Cholesky factorization of the covariance matrix of the observed data.
The algorithm is fast and easy to implement
and has theoretical grantees for exact recovery.
On synthetic and real-world datasets,
the algorithm is significantly faster than previous methods and achieves the state-of-the-art performance.
Furthermore, under the equal error variances assumption, we incorporate an optimization procedure into the Cholesky factorization based algorithm to handle the DAG recovering problem with latent variables.
Numerical simulations show that the modified ``Cholesky + optimization'' algorithm is able to
recover the ground truth graph in most cases and outperforms existing algorithms.\\


\end{abstract}

\newpage

\section{Introduction}\label{introduction}

Learning the causal relationships is a fundamental
problem and has many applications in biology, machine learning, medicine, and economics.
However, in many cases, not all causal variables relevant to the observed features have been measured. For example, in healthcare
domains, there may be numerous unobserved factors such
as gene expression; In financial markets, stock returns may
be causally related but may also be confounded or mediated by a complicated network of unmeasured
economic and political factors; Self-reported family history and diets
may also leave out some important information.

Most causal discovery approaches focus on the situation without latent variables. For example,
search-based algorithms~\citep{Chickering2002optimal,friedman2003being,teyssier2005ordering,aragam2015concave,
ramsey2017million,
tsamardinos2006max,lv2021bic,
ye2021optimizing} generally adopt a score (e.g., BIC score~\citep{peters2014causal}, Cholesky score~\citep{ye2021optimizing},  remove-fill score~\citep{squires2020efficient}, Clustering Information Criterion (CIC)~\citep{niu2022learning}) to measure the fitness of graphs over data and then search over the legal DAG (directed acyclic graph) space to find the structure that achieves the highest score. However, exhaustive search over the legal DAG space is infeasible when $p$ is large (e.g., there are $4.1 \mathrm{e}^{18}$ DAGs for $p = 10$~\citep{sloane2003line}).
Those algorithms go in quest of a trade-off between performance and time complexity.

Topology order search algorithms~\citep{ghoshal2017learning,
ghoshal2018learning,chen2019causal,gao2020polynomial,
park2020identifiability} decompose the DAG learning problem into two phases: (i) Topology order learning via conditional variance of the observed data; (ii) Graph estimation depends on the learned topology order.
Those algorithms reduce the computation complexity into polynomial time and are guaranteed to recover the DAG structure under proper assumptions.
Since Zheng {\em et al.}~\citep{zheng2018dags} proposed an approach that converts the traditional combinatorial optimization problem into a continuous program, many methods~\citep{ng2019graph,yu2019dag,lachapelle2020gradient,lee2020scaling,squires2020efficient,zheng2020learning,
zhu2021efficient,ng2022masked} have been proposed.
To overcome the constraints of conventional methodologies, researchers have introduced algorithms that leverage reinforcement learning~\citep{zhu2020causal, wang2021ordering} or flow-based techniques~\citep{ren2021causal, ren2022causal, ren2022variational, ren2022flow}, aiming to mitigate the stringent assumptions inherent in traditional approaches.

Causal discovery with latent variables is less studied.
The FCI~\citep{spirtes1999algorithm}, RFCI~\citep{colombo2012learning}, and ICD~\citep{rohekar2021iterative} methods can distinguish the observed variables and the latent variables that are confounders of the two observed ones and can recover up to a partial ancestral graph (PAG), which is an equivalence class of the true causal graph, even in the situation that the causal graph is able to exactly recovered.
Some methods focus on special topology cases for non-Gaussian models, e.g.,~\citep{spearman1928pearson,hoyer2008estimation,shimizu2009estimation,
tashiro2014parcelingam,cai2019triad}. Other works~\citep{salehkaleybar2020learning,xie2020generalized} relax the special cases to more general non-Gaussian partially observable DAGs.

\vspace{0.1in}
\noindent\textbf{Contributions:}
This paper studies the linear causal models via Cholesky factorization:
\begin{enumerate}
\item
For the case when all variables are observed,
we propose a Cholesky factorization based method (namely, CDCF) to discover the causal structure.
The time complexity of CDCF is $O(p^3)$,
which is the fastest method so far.
Here $p$ is the number of nodes.
\item We show that CDCF is able to recover the DAG exactly, under standard assumptions.
Sample complexity can be also obtained.  We propose a novel algorithm (namely, CDCF$^+$) to deal with the linear causal models with latent variables.
\item  Numerical simulations show that  CDCF and CDCF$^+$ achieve the state-of-the-art performance.
\end{enumerate}

\newpage

\noindent\textbf{Organization}. Section~\ref{sec:sem} introduces the algorithm for linear causal models with exact recovery analysis. Section~\ref{sec:latent} gives the algorithm for the linear causal models with latent variables.
Numerical  results are provided in Section~\ref{sec:exp}. Finally, Section~\ref{sec:conclusion} concludes the paper.

\vspace{0.15in}
\noindent{\bf Notations.}\; The symbol $\|\cdot\|$ stands for the Euclidean norm of a vector or the spectral norm of a matrix.
For a matrix $\mA\in\R^{m\times n}$,
$\mA_{ij}$, $\mA_{i,:}$, $\mA_{:,j}$ stand for the $(i,j)$ entry, the $i$th row and $j$th column of $\mA$, respectively.
Let $\underline{\vi}=[i_1,\dots,i_k]$, $\underline{\vj}=[j_1,\dots,j_{\ell}]$ be subsets of $[1,\dots,m]$, $[1,\dots,n]$, respectively.
$\mA_{\underline{\vi},:}$, $\mA_{:,\underline{\vj}}$, $\mA_{\underline{\vi},\underline{\vj}}$
stand for the sub-matrices of $\mA$ consisting of
all rows in $\underline{\vi}$, columns in $\underline{\vj}$, and the intersection of rows in $\underline{\vi}$ and columns in $\underline{\vj}$, respectively.
The symbol $\|\mA\|_1$, $\|\mA\|_{\max}$ $\|\mA\|_{2,\infty}$ stand for $\ell_1$-norm, max-norm and 2-infinity norm, respectively,
i.e., $\|\mA\|_1=\sum_{i,j}|\mA_{ij}|$,
$\|\mA\|_{\max}=\max_{i,j}|\mA_{ij}|$,
$\|\mA\|_{2,\infty}=\max_i\|\mA_{i,:}\|$.

\section{Recover Linear Causal Models via Cholesky Factorization}\label{sec:sem}
\vspace{-0.05in}

\subsection{Preliminaries and Assumptions}\label{sec:preliminaries}

Assume that the observed data is entailed by a DAG $\gG = (V, E)$, where $V$, $E$ are the sets of nodes and edges, respectively.
Each node $v_i$ corresponds to a random variable $\emX_i$.
The data matrix is given by $\mX = [\rvx_1,...,\rvx_p] \in \R^{n \times p}$, where $p=|V|$, $\rvx_i$ is consisted of $n$ i.i.d observations of the random variable $\emX_i$, for $k=1,\dots,p$.
The joint distribution of $\mX$ is $P(\mX) = \prod_{i=1}^p P(\emX_i| \parents_\gG(\emX_i))$, where $\parents_\gG(\emX_i)
:=\{\emX_j | (v_i, v_j) \in E\}$ is the set of parents of node $\emX_i$.

Given $\mX$, we seek to recover the latent DAG topology structure for the joint probability distribution~\citep{hoyer2008nonlinear,peters2017elements}.
Generally, $\mX$ is modeled via a structural equation model (SEM) with the form
\begin{equation*}
\emX_i = f_i(\parents_\gG(\emX_i)) + \emN_i, \quad  i=1,...,p,
\end{equation*}
where $f_i$ is an arbitrary function representing the relation between $\emX_i$ and its parents, $\emN_i$ is the jointly independent noise variable.

In this paper, we focus on the linear SEM defined by
\begin{equation}\label{model}
\emX_i = [\emX_1,\dots,\emX_p]\vw_i + \emN_i, \quad i=1,...,p,
\end{equation}
where $\rvw_i \in \R^p$ is a weighted column vector. Let
$\mW=[\vw_1,\dots,\vw_p] \in \mathbb{R}^{p \times p}$ be the weighted adjacency matrix,
$\mN = [\vn_1,\dots,\vn_p]\in\R^{n\times p}$ be the noise matrix,  where $\vn_i$ is $n$ i.i.d observations of the noise variable $\emN_i$. The linear SEM model \eqref{model} can be rewritten as
\begin{equation}\label{eq:matrix_model}
\mX = \mX \mW + \mN.
\end{equation}
As proposed in~\citet{nicholson1975matrices,mckay2004acyclic},
a graph is DAG if and only if the corresponding weighted adjacent matrix $\mW$ can be decomposed into
\begin{equation}\label{wptp}
\mW =\mP \mT \mP^{\T},
\end{equation}
where $\mP$ is a permutation matrix, $\mT$ is a strict upper triangular matrix, i.e., $\mT_{ij}=0$ for all $i\ge j$.


\vspace{0.1in}
\noindent{\bf Layer Decomposition}\;
The {\em layer decomposition}~\citep{gao2020polynomial} of a DAG can be obtained as follows: let $\gL_0:=\emptyset$, $\gA_j:=\cup_{m=0}^j \gL_j$;
for $j>0$, $\gL_j$ is the set of all source nodes
in the subgraph formed by removing the nodes in $\gA_{j-1}$.

\newpage
\begin{figure}[h]
\centering
\includegraphics[width=2.5in]{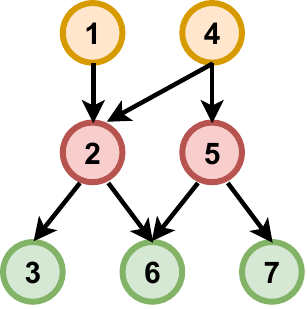}
\caption{Three-layer~DAG}
\label{fig_1}
\end{figure}

For example, the graph in Figure~\ref{fig_1} has three layers with $\gL_1=\{node 1, node 4\}$,
$\gL_2=\{node 2, node 5\}$,
$\gL_3=\{node 3, node 6, node 7\}$.
Furthermore, the adjacent matrix $\mW$ can be given by
\begin{equation*}
\mW=\begin{bmatrix} 0 & 1 & 0 & 0 & 0 & 0 & 0\\
0 & 0 & 1 & 0 & 0 & 1 & 0\\
0 & 0 & 0 & 0 & 0 & 0 & 0\\
0 & 1 & 0 & 0 & 1 & 0 & 0\\
0 & 0 & 0 & 0 & 0 & 1 & 1\\
0 & 0 & 0 & 0 & 0 & 0 & 0\\
0 & 0 & 0 & 0 & 0 & 0 & 0\end{bmatrix}.
\end{equation*}
Let
\begin{equation*}
\mP=\begin{bmatrix} 1 & 0 & 0 & 0 & 0 & 0 & 0\\
0 & 0 & 1 & 0 & 0 & 0 & 0\\
0 & 0 & 0 & 0 & 1 & 0 & 0\\
0 & 1 & 0 & 0 & 0 & 0 & 0\\
0 & 0 & 0 & 1 & 0 & 0 & 0\\
0 & 0 & 0 & 0 & 0 & 1 & 0\\
0 & 0 & 0 & 0 & 0 & 0 & 1\end{bmatrix},
\qquad
\mT=\begin{bmatrix} 0 & 0 & 1 & 0 & 0 & 0 & 0\\
0 & 0 & 1 & 1 & 0 & 0 & 0\\
0 & 0 & 0 & 0 & 1 & 1 & 0\\
0 & 0 & 0 & 0 & 0 & 1 & 1\\
0 & 0 & 0 & 0 & 0 & 0 & 0\\
0 & 0 & 0 & 0 & 0 & 0 & 0\\
0 & 0 & 0 & 0 & 0 & 0 & 0\end{bmatrix},
\end{equation*}
\eqref{wptp} holds.
Note that
$\mT$ is a strict {block} upper triangular matrix.
For a general $r$-layer DAG, there exists a permutation matrix $\mP$ such that \eqref{wptp} holds with $\mT$ having a strict block upper triangular form:
\begin{equation}\label{tblk}
\mT=[\mT_{ij}]
=\begin{bmatrix}
0 & \mT_{12} & \mT_{13} &\dots & \mT_{1r}\\
0 & 0 & \mT_{23} & \dots & \mT_{2r}\\
0 & 0 &  & \ddots & \vdots\\
\vdots &  & \ddots &  & \mT_{r-1,r}\\
0 & \dots & \dots & 0 & 0
\end{bmatrix},
\end{equation}
where $\mT_{ij}\in\R^{p_i\times p_j} = 0$ for $1\le i\le j\le r$, $p_i$ equals the number of nodes in $\gL_i$ for $1\le i\le r$, each column of $\mT_{i,i+1}$ is nonzero (otherwise, the corresponding node belongs to the $i$th layer).
The decomposition \eqref{wptp} with $\mT$ having the strict block upper triangular form \eqref{tblk}
is {\em unique} up to all within layer permutations (the Markov equivalence class).

Throughout the rest of this paper, we will make the following assumptions:

\vspace{0.1in}

\noindent{\bf A1} The noise variables are all centered, uncorrelated, and have finite variances, i.e.,
\begin{equation*}
\E(\emN_i)=0,\;
\mSigma_{nn}=[\E(\emN_i\emN_j)]=\diag(\sigma_1^2,\dots,\sigma_p^2)<\infty.
\end{equation*}

\noindent{\bf A2}
Let the weighted adjacency matrix $\mW$ be decomposed as in \eqref{wptp} with $\mT$ having the block diagonal form \eqref{tblk}.
Assume
\begin{equation*}
\Delta=\min_{\substack{j<k\\  i_s\in\gI_j,i_t\in\gI_k}}
\sigma_{i_t}^2-\sigma_{i_s}^2 +
\|[\mU]_{i_s:i_t-1,i_t}\|^2>0,
\end{equation*}
where
\begin{subequations}
\begin{align}
\underline{\vi}&=[i_1,\dots,i_p]=[1,\dots,p]\mP,\label{vi}\\
\gI_j&=\{\sum_{i=1}^{j-1} p_i+1,\dots,\sum_{i=1}^{j} p_i\}\quad \mbox{for}\; 1\le j\le r,\label{Ij}\\
\mU &= \diag(\sigma_{i_1},\dots,\sigma_{i_p}) (\mI - \mT)^{-1}.\label{mU}
\end{align}
\end{subequations}


Assumption {\bf A2} is essentially the same as the identifiability condition in~\citet[Thm. 2]{park2020identifiability} and~\citet[Thm. B2]{gao2020polynomial}.

In the rest of this section, we assume that  there are no latent variables.
The case when latent variables exist will be discussed in Section~\ref{sec:latent}.

\subsection{Algorithm}\label{sec:alg1}
It follows from \eqref{eq:matrix_model}, \eqref{wptp} and \eqref{vi} that
\begin{equation}\label{eq:PX}
[\emX_{i_1},\dots,\emX_{i_p}] = [\emN_{i_1},\dots,\emN_{i_p}] (\mI-\mT)^{-1}.
\end{equation}
Using \eqref{eq:PX} and {\bf A1}, we have
\begin{equation}
[\mSigma_{xx}]_{\underline{\vi},\underline{\vi}}
=(\mI - \mT)^{-\T}[\mSigma_{nn}]_{\underline{\vi},\underline{\vi}} (\mI - \mT)^{-1}
=\mU^{\T}\mU,
\label{chol}
\end{equation}
where $\mU$ is defined in \eqref{mU}.
Note that $\mU$ is upper triangular,
hence \eqref{chol} is simply
the Cholesky factorization of the permuted covariance matrix $[\mSigma_{xx}]_{\underline{\vi},\underline{\vi}}$.
The task becomes to find the permutation $\underline{\vi}=[i_1,\dots,i_p]$ and an upper  triangular matrix $\wtd\mU$ such that
$\wtd\mU^{\T}\wtd\mU$ is a good approximation of
$[\mSigma_{xx}]_{\underline{\vi},\underline{\vi}}$.

Let $\wtd\mSigma_{xx}$ be an estimator for the population covariance matrix $\mSigma_{xx}$.
Assume that $\underline{\vi}_{k-1}:=[i_1,\dots,i_{k-1}]$ and $\wtd\mU_{k-1}:=\wtd\mU_{1:k-1,1:k-1}$ are settled,
and we have
\begin{equation}
[\wtd\mSigma_{xx}]_{\underline{\vi}_{k-1},\underline{\vi}_{k-1}}
= \wtd\mU_{k-1}^{\T}\wtd\mU_{k-1}.\label{ck1}
\end{equation}
Next, we show how to find $i_k$ and $\wtd\mU_k$.
For the time being, let us assume $i_k$ is known, we show how to compute the last column of $\wtd\mU_k$.
Let $\wtd\mU_k= \begin{bmatrix} \wtd\mU_{k-1} & \bm\beta_k \\ 0 & \alpha_k \end{bmatrix}$ with $\alpha_k\in\R$, $\bm\beta_k\in\R^{k-1}$.
Then using \eqref{ck1}, we get
\begin{equation*}
[\wtd\mSigma_{xx}]_{\underline{\vi}_k,\underline{\vi}_k}=
\wtd\mU_k^{\T} \wtd\mU_k
= \begin{bmatrix} \wtd\mU_{k-1}^{\T}\wtd\mU_{k-1} & \wtd\mU_{k-1}^{\T}\bm\beta_k \\ \bm\beta_k^{\T}\wtd\mU_{k-1} & \alpha_k^2+\|\bm\beta_k\|^2 \end{bmatrix}.
\end{equation*}
It follows immediately that
\begin{equation}\label{eq:cholesky_diag}
\bm\beta_k = \wtd\mU_{k-1}^{-\T} [\wtd\mSigma_{xx}]_{\underline{\vi}_{k-1},i_k},\quad
\alpha_k = \sqrt{[\wtd\mSigma_{xx}]_{i_k,i_k}  - \|\bm\beta_k\|^2}.
\end{equation}
By \eqref{eq:cholesky_diag}, once $i_k$ is settled, we can obtain the last column of $\wtd\mU_k$.
The task remains to select $i_k$ from $\{1,\dots,p\}\setminus\{i_1,\dots,i_{k-1}\}$.
In this paper, we compute $\alpha_j$ and $\bm\beta_j$ by \eqref{eq:cholesky_diag}
for all possible $j$ ($i_j\in\{1,\dots,p\}\setminus\{i_1,\dots,i_{k-1}\}$).
Then determine $i_k$ by
\begin{equation}\label{ik}
i_k=\argmin_{k\le j \le p} \alpha_j^2.
\end{equation}
The intuition behind \eqref{ik} is that $\alpha_j^2$ is the empirical conditional variance of $\emX_j$, thus, $\mX_{i_k}$ is chosen as the one with the smallest conditional variance. One may also consider some alternatives, say choose $i_k$ via the sparsity of $\wtd\mU_{k-1}^{-1}\bm\beta_k$, or take both the sparsity and the conditional variance into account.


\begin{algorithm}[t!]
\caption{Causal Discovery via Cholesky Factorization (CDCF)}
\begin{algorithmic}[1]
\State {\bfseries input}: The covariance matrix $\wtd\mSigma_{xx}\in\R^{p\times p}$.

\State {\bfseries output}: A permutation $\wtd\mP$
and an upper triangular matrix $\wtd\mU_p^{-1}$.

\State Set $\underline{\vi} = [1,2,\dots,p]$;

\State
$\ell = \argmin_{1\le j\le p} [\wtd\mSigma_{xx}]_{i_j,i_j}$,
exchange $i_1$ and $i_\ell$ in $\underline{\vi}$;
\State Set $\alpha_1 = \sqrt{[\wtd\mSigma_{xx}]_{i_1,i_1}}$,
$\wtd\mU_1=\alpha_1$,
$\wtd\mU_1^{-1}=\frac{1}{\alpha_1}$;

\For{$k=2,3,\dots, p$}
\For{$j=k,k+1,\dots,p$}
\State
Compute $\alpha_j$, $\bm\beta_j$ via \eqref{eq:cholesky_diag};
\EndFor
\State $\ell=\argmin_{k\le j \le p} \alpha_j^2$, exchange $i_k$ and $i_{\ell}$ in $\underline{\vi}$;

\State $\wtd\mU_k = \begin{bmatrix} \wtd\mU_{k-1} & \bm\beta_{\ell}\\  0 & \alpha_{\ell}\end{bmatrix}$,
$\wtd\mU_k^{-1} = \begin{bmatrix} \wtd\mU_{k-1}^{-1} & -\frac{1}{\alpha_{\ell}}\wtd\mU_{k-1}^{-1}\bm\beta_{\ell}\\  0 & \frac{1}{\alpha_{\ell}}\end{bmatrix}$;
\EndFor
\State Set $\wtd\mP=\mI(:,\underline{\vi})$, $\mI$ is the order $p$ identity matrix.
\end{algorithmic}
\label{alg1}
\end{algorithm}

The overall algorithm is summarized in Algorithm~\ref{alg1}.
Recall \eqref{mU} and denote $\mA=\diag(\alpha_1,\dots,\alpha_p)$,
we can obtain the adjacent matrix $\wtd\mW$ as follows:
\begin{equation}\label{adj}
\wtd\mW = [\textsc{triu}(\textsc{truncate}(\wtd\mU_p^{-1}\mA,\omega))]_{\textsc{rev}(\underline{\vi}),\textsc{rev}(\underline{\vi})},
\end{equation}
that is, we truncate $\wtd\mU_p^{-1}\mA$ element-wisely,
then take its strict upper triangular part (denoted by ``\textsc{triu}'') and re-permute the predicted adjacent matrix back to the original order according to the permutation order $\underline{\vi}$.
Here $\omega>0$ is a truncate threshold:
a number less than $\omega$ in absolute value is truncated to zero; otherwise, it remains unchanged.

\vspace{0.1in}
\noindent{\bf Time Complexity.}
On line 8, we only need to compute the last entry of $\bm\beta_j$ at the cost of $\mathcal{O}(p)$ at worst since the other entries are available from the previous steps. Therefore, the overall time complexity of Algorithm~\ref{alg1} is $\mathcal{O}(p^3)$.
Additionally, the inner loop (lines 7 to 9) can be made in parallel, which makes the algorithm friendly to run on GPU and suitable for large scale calculations.
It is worth mentioning here that the time complexity for the computation of $\wtd\mSigma_{xx}$ is not included here
since we take it as an overhead of CDCF.

\begin{remark}
Algorithm~\ref{alg1} differs from other methods
(e.g.,~\citet{ghoshal2017learning,
ghoshal2018learning,chen2019causal,gao2020polynomial,
park2020identifiability})
in that, it combines the topology order search stage and the graph estimation stage together, rather than two-stage: first topology order search then graph estimation,
which leads to the reduction of the time complexity.
\end{remark}

\subsection{Exact Recovery}
In this section, we show that Algorithm~\ref{alg1} is able to exactly recover the ordering and the graph structure. All proofs are given in the appendix.

\begin{theorem}[Exact recovery of the ordering]\label{the:recover}
For the linear SEM model \eqref{model},
assume \eqref{wptp}, {\bf A1} and {\bf A2}.
Let $\wtd\mSigma_{xx}$ be an estimator for $\mSigma_{xx}=[\E(\emX_j\emX_k)]$ and
$\max_i|[\mSigma_{xx}]_{ii}-[\wtd\mSigma_{xx}]_{ii}|\le \epsilon_d$
for some $\epsilon_d>0$.
If
$\epsilon_d < \frac{\Delta}{4}$,
Algorithm~\ref{alg1} recovers the ordering exactly.
\end{theorem}

\vspace{0.1in}
\noindent{\bf Sample complexity for the ordering.}\;
According to Theorem~\ref{the:recover},
we can obtain the sample complexity for the exact recovery of the ordering by
$\max_i\|[\mSigma_{xx}]_{ii}-[\wtd\mSigma_{xx}]_{ii}\|\le \Delta/4$.
Specifically, $X_i$ is a linear combination of $\emN_i$'s and $\E(\emX_i)=0$,
$\mbox{Var}(\emX_i)=[\mSigma_{xx}]_{ii}$.
Then we can obtain the sample complexity $\gO(\log(p/\epsilon)\gM^2)$, where $\gM=\max_i[\mSigma_{xx}]_{ii}$. The proof can be found in Appendix.
In the following Table~\ref{tab:sample_complexity}, we compare the sample complexity with others'.

\begin{table*}[h]
\begin{center}
\caption{Sample complexity  for the ordering. The last column represents the $\mathcal{O}$ complexity of the sample number $n$ that makes the algorithm recover the DAG with probability at least $1 - \epsilon$, $p$ is the nodes number, $r$ represents the level of the graph, $q$ is the maximum in-degree, $d$ is the maximum total degree, $m$ represents the $m$'th moment bounded noise, $\mathcal{M} = \max_i[\mSigma_{xx}]_{ii}$,
$g(x) = x/ \log x$ ($g^{-1}$ exists when $x>3$ and it holds $g^{-1}(x) > x$).
\label{tab:sample_complexity}}
\setlength{\tabcolsep}{0.5mm}{
\begin{tabular}{c|c|c|c|c|c}
\hline\hline
\bf Algorithm & \bf Data & \bf Function & \bf Noise Type & \bf $\mX$ &\bf Sample Complexity \\  \hline

\makecell{NPVAR\\ ~\citet{gao2020polynomial}} &-&\makecell{(Non)-linear\\ Lip-continuous} &-&-&$\mathcal{O}((rp/\epsilon)^{1+p/2})$ \\ \hline

\multirow{2}{*}{\makecell{EV\\ ~\citet{chen2019causal}}}
& $n>p$& Linear & Sub-Gaussian &\makecell{$\lambda_{\text {min}} > 0$ }& $\mathcal{O}( p^2\log(p/\epsilon) \mathcal{M}^2)$  \\\cline{2-6}
& $n<p$ & Linear & Sub-Gaussian & \makecell{$\lambda_{\text {min}} > 0$ }& $\mathcal{O}(q^2\log(p/\epsilon) \mathcal{M}^2)$   \cr \hline

\multirow{2}{*}{\makecell{LISTEN\\ ~\citet{ghoshal2018learning}}}
&-& Linear & \makecell{Sub-Gaussian} & -  &$\mathcal{O}(d^4\log(\frac{p}{\sqrt{\epsilon}})\mathcal{M}^2)$\\ \cline{2-6}
&-& Linear & \makecell{Bounded moment} &-& $\mathcal{O}(d^4(\frac{p^2}{\epsilon})^{1/m}\mathcal{M}^2)$\\ \hline

\makecell{US\\ ~\citet{park2020identifiability}}
& $n>p$ & Linear & \makecell{Gaussian} & \makecell{$\lambda_{\min} > 0$\\ $\lambda_{\max}<\infty$}&$g^{-1}(\mathcal{O}(\log(p/\epsilon)\mathcal{M}^2))$  \\ \hline

CDCF (ours) & - &Linear& - & - &
{$\mathcal{O}(\log(p/\epsilon)\gM^2)$}\\
\hline\hline
\end{tabular}
}
\end{center}
\end{table*}

\begin{theorem}[Exact recovery of the structure]\label{the:recover2}
Follow the notations and assumptions in Theorem~\ref{the:recover}.
Denote
\begin{equation*}
\mu=\|\mI-\mT\|_{2,\infty}\|\mI-\mT^{\T}\|_{2,\infty},
\quad
\rho = \frac{\max_i\sigma_i}{\min_i\sigma_i^2},
\quad
\omega=\min_{\mT_{ij}\ne 0}|\mT_{ij}|.
\end{equation*}
Assume
$\|\mSigma_{xx}-\wtd\mSigma_{xx}\|\le \epsilon_2$ for some $\epsilon_2>0$.
If
$\epsilon_2\lesssim\frac{\omega^2}{8\rho^2\mu^2}$,
then Algorithm~\ref{alg1} is able to
 recover the graph structure exactly.
\end{theorem}

\noindent{\bf Sample complexity for the graph structure.}\;
The sample complexity for the exact recovery of the graph structure can be established via $\|\mSigma_{xx}-\wtd\mSigma_{xx}\| \lesssim \frac{\omega^2}{8\rho^2\mu^2}$.
For example, when the noise is sub-Gaussian, it holds with probability at least $1-2\exp(-ct^2)$ that
\begin{equation*}
\|\wtd\mSigma_{xx}-\mSigma_{xx}\|\le \|(\mI-\mT)^{-1}\|^2\max\{\delta,\delta^2\},
\end{equation*}
where $\delta=C\sqrt{\frac{p}{n}}+\frac{t}{\sqrt{n}}$, $c$ and $C$ are two constants.
For simplicity, let $\delta\ge 1$.
By calculations, we know that w.p. $\ge 1-\epsilon$ it holds  $\|\mSigma_{xx}-\wtd\mSigma_{xx}\| \lesssim \frac{\omega^2}{8\rho^2\mu^2}$ for $n=\gO((p+\log\frac{1}{\epsilon}) \rho^2\mu^2 \|(\mI-\mT)^{-1}\|^2/\omega^2)$.
Sample complexities for other distributions of the noise can be  obtained similarly.


\section{Learning Linear Causal Models with Latent Variables}\label{sec:latent}

In this section, we study the effects of latent variables, then present an algorithm to recover them.

\subsection{The Effects of Latent Variables}
When there are latent variables, the covariance matrix of the observed variables is essentially a principal submatrix of the overall covariance matrix.
To get a clue on how the Cholesky factor changes, we consider the case when there is one latent variable.

\vspace{0.1in}

Let $\mSigma_+=\begin{bmatrix}\mSigma_{11} & \vx &  \mSigma_{12}\\ \vx^{\T} & d^2 & \vy^{\T}\\ \mSigma_{12}^{\T} &
\vy & \mSigma_{22} \end{bmatrix}$ be a symmetric positive definite matrix, where $\vx\in\R^{p_1}$, $\vy\in\R^{p_2}$, $\mSigma_{11}\in\R^{p_1\times p_1}$,
$\mSigma_{22}\in\R^{p_2\times p_2}$.
Let the Cholesky factorization of $\mSigma_+$ be $\mSigma_+=\mU_+^{\T}\mU_+$,
where $\mU_+=\begin{bmatrix} \mU_{11} & \wht\vx & \mU_{12}\\
0 & \wht d & \wht\vy^{\T} \\ 0 & 0 & \wht\mU_{22}\end{bmatrix}$,
$\mU_{11}\in\R^{p_1\times p_1}$ and
$\wht\mU_{22}\in\R^{p_2\times p_2}$ be both upper triangular.
By simple calculations, we have
\begin{equation}
\wht\vx=\mU_{11}^{-\T}\vx,
\wht d = \sqrt{d^2-\|\wht\vx\|^2},
\wht\vy = \frac{1}{\wht d}(\vy-\mU_{12}^{\T}\wht\vx).
\end{equation}
Denote
$\vv=\mU_{11}^{-1}\wht\vx$, 
$\vw=\wht\mU_{22}^{-\T}\wht\vy$.
It also holds that
\begin{equation}
\mU_+^{-1}= \begin{bmatrix}\mU_{11}^{-1} & -\frac{1}{\wht d}\vv &\frac{1}{\wht d} \vv \vw^{\T} -\mU_{11}^{-1}\mU_{12}\wht\mU_{22}^{-1}\\
0 & \frac{1}{\wht d} & -\frac{1}{\wht d} \vw^{\T}\\
0 & 0 &\wht\mU_{22}^{-1}\end{bmatrix}.
\end{equation}

Now we remove the $p_1+1$st row and the corresponding column of $\mSigma_+$, we get $\mSigma=\begin{bmatrix}\mSigma_{11} & \mSigma_{12}\\ \mSigma_{12}^{\T} & \mSigma_{22} \end{bmatrix}$.
For the Cholesky factor of $\mSigma$, we have the following results.

\newpage
\begin{proposition}\label{prop1}
Let the Cholesky factorization of $\mSigma$ be $\mSigma=\mU^{\T}\mU$. It holds that

\vspace{0.1in}

\noindent{\bf (a)}\;
$\mU$ can be given by $\mU=\begin{bmatrix} \mU_{11} &\mU_{12} \\ 0& \mU_{22}\end{bmatrix}$, where $\mU_{22}$ satisfies
\begin{equation}\label{uu4}
\mU_{22}^{\T}\mU_{22}=\wht\mU_{22}^{\T}\wht\mU_{22}+\wht\vy\wht\vy^{\T}.
\end{equation}

\noindent{\bf (b)}\;
Let the Cholesky factorization of
$\mI+\vw\vw^{\T}$ be
$\mI+\vw\vw^{\T}=\mL^{\T}\mL$.
Then $\mU^{-1}$ can be given by
\begin{equation}
\mU^{-1}=\begin{bmatrix} \mU_{11}^{-1} & -\mU_{11}^{-1}\mU_{12}\mU_{22}^{-1}\\ 0 & \mU_{22}^{-1}\end{bmatrix}, \;\; \mU_{22}^{-1}=\wht\mU_{22}^{-1}\mL^{-1}.
\end{equation}

\noindent{\bf (c)}\;
The $k$th diagonal entry of $\mU_{22}$ can be given by
\begin{equation*}
[\mU_{22}]_{kk} = \sqrt{(1+\|\vw_{1:k}\|^2)/(1+\|\vw_{1:k-1}\|^2)}[\wht\mU_{22}]_{kk}.
\end{equation*}
\end{proposition}

\vspace{0.1in}

Proposition~\ref{prop1} (a) tells that the $(1,1)$ and $(1,2)$ blocks of $\mU$ are the same as the $(1,1)$ and $(1,3)$ blocks of $\mU_+$, respectively.
In addition, \eqref{uu4} is a rank-1 update of Cholesky factorization.
Proposition~\ref{prop1} (b) tells how the blocks of $\mU^{-1}$ relate with the blocks of $\mU_+^{-1}$.
The $(1,1)$ block of
$\mU^{-1}$ is the same as the $(1,1)$ block of $\mU_+^{-1}$.
Recall that $\vv$ and $\vw$ determine the in-going and out-going edges of the $p_1+1$st node, respectively, we expect $\vv$ and $\vw$ to be sparse.
For all $1\le i<j<p_2$, $\mG_{ij}\ne 0$ if and only if $\vw_i\vw_j\ne 0$. Hence $\mG$ (also its $\mG^{-1}$) are sparse.
We  thus claim that
the $(1,2)$ block of $\mU^{-1}$ can be obtained from the $(1,3)$ block of $\mU_+^{-1}$ by a sparse rank-1 update followed by a sparse matrix multiplication;
the $(2,2)$ block of $\mU^{-1}$ can be obtained from the $(3,3)$ block of $\mU_+^{-1}$ by a sparse matrix multiplication.
Proposition~\ref{prop1} (c) tells that
whenever $\vw_k\ne 0$,
$[\mU_{22}]_{kk}>[\wht\mU_{22}]_{kk}$
since
$\frac{1+\|\vw_{1:k}\|^2}{1+\|\vw_{1:k-1}\|^2}>1$.
This indicates
 the variance of the child node becomes larger when its parent nodes are missing.

\vspace{0.1in}

According to Proposition~\ref{prop1} and the discussion above, we may make the following claim:
\begin{itemize}
\item {\bf C1}\; The $(1,2)$ and $(2,2)$ blocks of $\mU^{-1}$ become denser compared with the $(1,3)$ and $(3,3)$ blocks of $\mU_+^{-1}$;
\item {\bf C2}\; The variance for the child node of latent nodes becomes larger.
\end{itemize}

In order to recover DAG with latent variables, we have to detect the latent variables and identify their in-going and out-going edges.
In this paper, we propose to detect latent variables via {\bf C2} and identify in-going and out-going edges via {\bf C1}. Note that
{\bf C2} alone is insufficient to detect latent variables since we only know the variance becomes larger.
Thus, we make an equi-variance assumption to  simplify the problem:

\vspace{0.1in}

\noindent{\bf A3}
The noise variables are all centered, uncorrelated and have equal variance $\sigma^2$.

\begin{remark}
The essence of detecting latent variables is to determine where a latent node is missing and how it is connected to other nodes.
To do so,
one need to carefully design some criterion:
whenever the criterion is violated, there is latent nodes.
{\bf A3} is perhaps the simplest criterion to accomplish the task.
One may also consider exploring various alternatives, such as employing graph structures to refine the criterion. This approach might involve focusing on the sparsity pattern of the graph or examining the in-degree and out-degree of the nodes within the graph.
\end{remark}

\subsection{Algorithm}

To proceed, we assume  an estimation $\wht\sigma$ for $\sigma$ in {\bf A3} is available.   This assumption is not restrictive since when a root node is observed, we can estimate $\wht\sigma$ via the sample variance of the root node.

At beginning, we observe $q<p$ nodes.
Input $\wtd\mSigma_{xx}\in\R^{q\times q}$
into Algorithm~\ref{alg1}, then it works as there were no latent variables until it encounters a child of latent variables.
So, we may detect a latent variable
when we find a diagonal entry of $\wtd\mU_q$ larger than $\wht\sigma$, then we may determine its connection with other nodes via a sparse encouraging optimization problem.
The above procedure can be used repeatedly to handle the case when there are multiple latent variables.

Next, we present the flowchart of the algorithm:
\begin{itemize}
\item {\bf S1}\; Perform Algorithm~\ref{alg1} with input $\wtd\mSigma_{xx}\in\R^{q\times q}$, output $\wtd\mP$ and $\wtd\mU_q^{-1}$.
\item {\bf S2}\; Check the diagonal entries of $\wtd\mU_q^{-1}$: if all entries are approximately $\frac{1}{\sigma}$, then there is no (identifiable) latent variable;
otherwise, find the first entry that is not approximately $\frac{1}{\sigma}$ and let it be the $j$th.
\item {\bf S3}\; Insert a variable between $j-1$st and $j$th variables
and determine its connections with others.
\item {\bf S4}\; Update $\wtd\mSigma_{xx}$ by appending one row and column to it, then go to step {\bf S1} to repeat.
\end{itemize}

In what follows we give details for {\bf S3} and {\bf S4}.
Define
\begin{align}
\mU(\mS)&:=\wht\sigma (\mI -\mS)^{-1},\
\mC(\mS):=\mU(\mS)^{\T}\mU(\mS),  \label{ucs}\\
\mS\in\mathbb{U}_q&:=\{\mU\in\R^{q\times q}\;|\; \mU_{ij}=0 \mbox{ for all } i>j\}.\notag
\end{align}
Let $\underline{\vi}=[1,\dots,q]\wtd\mP$ and
we parameterize the Cholesky factor of $[\wtd\mSigma_{xx}]_{\underline{\vi},\underline{\vi}}$ by $\mU(\mS)$, with $\mS\in\mathbb{U}_q$.
Note that all diagonal entries of $\mU(\mS)$ are $\wht\sigma$.
The reason for having such a requirement is that
after inserting a variable, we expect that the variances of all variables are approximately $\sigma^2$.
Note that $\mS$ should be sparse, and the resulting covariance matrix should be close to the observed covariance matrix.
We thus determine $\mS$ by
\begin{equation}\label{opt}
\min_{\mS\in\mathbb{U}_q}\|[\mC(\mS)]_{\underline{\vj}^c,\underline{\vj}^c} - [\wtd\mSigma_{xx}]_{\underline{\vi}^c,\underline{\vi}^c}\|_F^2 + \mu \|\mS\|_1,
\end{equation}
where $\underline{\vi}^c$ and $\underline{\vj}^c$ are two index sets that pick all observed indices of $\wtd\mSigma_{xx}$ and $\mC(\mS)$, respectively, $\mu\ge 0$ is a parameter.

Let $\mS$ be the solution to \eqref{opt}.
We compute $\mC(\mS)$ by \eqref{ucs}, then update $\wtd\mSigma_{xx}$ as follows:
\begin{equation}\label{update}
\wtd\mSigma_{xx}=\begin{bmatrix} \wtd\mSigma_{xx} & \vz\\ \vz^{\T} & d^2\end{bmatrix},
\end{equation}
where $\vz=[\vx^{\T}, \vy^{\T}]^{\T}$ with $\vx=[\mC(\mS)]_{1:j-1,j}$ and $\vy=[\mC(\mS)]_{j+1:q,j}$,
$d^2=[\mC(\mS)]_{j,j}$.

\begin{algorithm}[t]
\caption{Causal Discovery with Latent Variables via Cholesky Factorization (CDCF$^+$)}
\begin{algorithmic}[1]
\State {\bfseries input}:
The covariance matrix $\wtd\Sigma_{xx}\in\R^{q\times q}$, an estimation $\wht\sigma$, and
two parameters $\zeta$, $\mu$.
\State {\bfseries output}: A permutation $\wtd\mP$ and an upper triangular matrix $\wtd\mU_q^{-1}$.
\State $\underline{\vj}=[]$;
\While{$\# \underline{\vj} \le p-q$}
\State {\bf call} Alg.~\ref{alg1} with input $\wtd\mSigma_{xx}$, output $\underline{\vi}$ and $\wtd\mU_q^{-1}$;

\If{$\min_i[\wtd\mU_q^{-1}]_{ii}<\frac{1-\zeta}{\wht\sigma}$}
\State
$j=\argmin_i[\wtd\mU_q^{-1}]_{ii}<\frac{1-\zeta}{\wht\sigma}$;
\State $q=q+1$;
\State $\underline{\vj}=[\underline{\vj},j]$, $\underline{\vj}^c=[1,\dots,q]\setminus\underline{\vj}$;
\State $\underline{\vi}=[i_1,\dots,i_{j-1},q,i_j,\dots,i_{q-1}]$, $\underline{\vi}^c=\underline{\vi}_{\underline{\vj}^c}$;
\State Solve \eqref{opt} for $\mS$;
\State Compute $\mC(\mS)$ in \eqref{ucs};
\State Update
$\wtd\mSigma_{xx}$ via \eqref{update};
\Else
\State {\bf break};
\EndIf
\EndWhile
\State Set $\wtd\mP=\mI(:,\underline{\vi})$. 
\end{algorithmic}
\label{alg2}
\end{algorithm}

The detailed algorithm is summarized in Algorithm~\ref{alg2}.
On input,
$\zeta$ is a threshold to determine if the diagonal entries of $\wtd\mU_q^{-1}$ are approximately $\frac{1}{\wht\sigma}$,
$\mu$ is the sparsity encouraging parameter used in \eqref{opt}.
Line 4, $\#\underline{\vj}$ stands for the number of elements in $\underline{\vj}$, $p$ is an integer no less than $q$. On output, the size of $\wtd\mU_q^{-1}$ is at most $p\times p$,
the diagonal entries of $\wtd\mU_q^{-1}$ are all approximately $\frac{1}{\wht\sigma}$. 

\newpage
\noindent{\bf Time Complexity.}
For each ``while'' loop of CDCF$^+$,
the time complexity is dominated by line 11;
line 5 is simply CDCF and the time complexity is $\gO(q^3)$;
line 12 is the inversion of an upper triangular matrix and a matrix-matrix multiplication, the time complexity is $\gO(q^3)$;
the time complexity of the rest can be ignored.
Line 11, sub-gradient method is used to solve the optimization problem \eqref{opt},
the time complexities for gradient calculation and updating $\mathbf{S}$ are both $\gO(q^2)$.
Assume the sub-gradient method needs $\gO(\log\frac1\epsilon)$ iterations to solve \eqref{opt}, then the time complexity for Line 11 is $\gO(q^2 \log\frac1\epsilon)$.
Overall speaking, the time complexity of CDCF$^+$ is
$\gO((p-q)(q + \log\frac1\epsilon)q^2)$.

\vspace{0.1in}
\noindent{\bf Convergence.}
With a proper choice of $\mu$, Algorithm~\ref{alg2} converges in finite steps, and the output $\wtd\mU_q^{-1}$ has almost identical diagonal entries, additionally, due to the $\ell_1$-penalty term, $\wtd\mU_q^{-1}$ is approximately sparse. The proof  can be found in the appendix.

\begin{remark}
Algorithm~\ref{alg2} produces a sparse DAG, but not necessarily the true DAG, since the true DAG with latent variables can be unidentifiable~\citep{adams2021identification}.
\end{remark}

\section{Numerical Experiments}\label{sec:exp}

In the experiments, the augmented covariance matrix
$\wtd\mSigma_{xx}=\frac1n \mX^{\T}\mX + \lambda \mI$ is used to estimate $\mSigma_{xx}$,
where $\lambda$ is a parameter.

\subsection{Experiments for CDCF}
In this section, we present the results of our experiments conducted on simulated graphs, bioinformatics datasets, and knowledge base datasets. Additionally, further experiments addressing non-Gaussian distributions and diagonal augmentation settings are provided in the appendix.
\subsubsection{Simulated Graphs}
We evaluate CDCF on simulated graphs from two well-known ensembles of random graph types: Erd\"os–R\'enyi (ER)~\citep{gilbert1959random} and Scale-free (SF)~\citep{barabasi1999emergence}.
The average edge number per node is denoted after the graph type, e.g., ER2 represents two edges per node on average.
After the graph structure is settled, we assign uniformly random edge weights. 
We generate the observation data $\mX$ from the linear Gaussian SEM.

Our baselines include: NOTEARS~\citep{zheng2018dags}, DAG-GNN~\citep{yu2019dag}, CORL~\citep{wang2021ordering}, NPVAR~\citep{gao2020polynomial}, and EQVAR~\citep{chen2019causal}. Other methods such as PC algorithm~\citep{spirtes2020causation}, LiNGAM~\citep{shimizu2006linear}, FGS~\citep{ramsey2017million}, MMHC~\citep{tsamardinos2006max}, L1OBS~\citep{schmidt2007learning}, CAM~\citep{buhlmann2014cam}, RL-BIC2~\citep{zhu2020causal}, A*LASSO~\citep{xiang2013a}, LISTEN~\citep{ghoshal2018learning}, US~\citep{park2020identifiability} perform no better than the baseline EQVAR.

\begin{table*}[h]
\centering
\caption{Results of 50, 100, 1000 nodes on 3000 linear Gaussian SEM samples.\vspace{-0.1in}}
\label{tab:baseline_results}
        \begin{tabular}{lccccccc}
            \hline\hline
            \ttfamily  \bf Nodes & \ttfamily \bf Graph & \ttfamily \bf NOTEARS & \ttfamily \bf DAG-GNN &
            \ttfamily \bf CORL-2 & \ttfamily \bf NPVAR & \ttfamily \bf EV-TD & \ttfamily \bf CDCF \\
            \hline

            \multirow{4}{*}{50} &
            \multirow{1}{*}{ER2} &

            \multicolumn{1}{c}{$38.6_{10.8}$} &
            \multicolumn{1}{c}{$30.6_{8.3}$} &
            \multicolumn{1}{c}{$17.9_{10.6}$} &
            \multicolumn{1}{c}{$0.4_{0.5}$} &
            \multicolumn{1}{c}{$\bf0.0_{0.0}$} &
            \multicolumn{1}{c}{$\bf0.0_{0.0}$} \\
            & \multirow{1}{*}{ER5} &

            \multicolumn{1}{c}{$67.8_{7.5}$} &
            \multicolumn{1}{c}{$93.2_{109.4}$} &
            \multicolumn{1}{c}{$64.8_{13.1}$} &
            \multicolumn{1}{c}{$0.6_{0.8}$} &
            \multicolumn{1}{c}{$0.1_{0.3}$} &
            \multicolumn{1}{c}{$\bf0.0_{\bf0.0}$} \\
            & \multirow{1}{*}{SF2} &

            \multicolumn{1}{c}{$3.5_{1.6}$} &
            \multicolumn{1}{c}{$79.3_{93.2}$} &
            \multicolumn{1}{c}{$\bf0.0_{\bf0.0}$} &
            \multicolumn{1}{c}{$1.1_{1.0}$} &
            \multicolumn{1}{c}{$\bf0.0_{\bf0.0}$} &
            \multicolumn{1}{c}{$\bf0.0_{\bf0.0}$} \\
            & \multirow{1}{*}{SF5} &

            \multicolumn{1}{c}{$20.1_{14.3}$} &
            \multicolumn{1}{c}{$89.2_{99.2}$} &
            \multicolumn{1}{c}{$20.8_{10.1}$} &
            \multicolumn{1}{c}{$1.0_{0.9}$} &
            \multicolumn{1}{c}{$\bf0.0_{\bf0.0}$} &
            \multicolumn{1}{c}{$\bf0.0_{\bf0.0}$}
            \\\specialrule{.075em}{.0em}{.0em}
            \multirow{4}{*}{100} &
            \multirow{1}{*}{ER2} &

            \multicolumn{1}{c}{$72.6_{23.5}$} &
            \multicolumn{1}{c}{$66.2_{19.2}$} &
            \multicolumn{1}{c}{$18.6_{5.7}$} &
            \multicolumn{1}{c}{$2.1_{1.2}$} &
            \multicolumn{1}{c}{$\bf0.0_{\bf0.0}$} &
            \multicolumn{1}{c}{$\bf0.0_{\bf0.0}$} \\
            & \multirow{1}{*}{ER5} &

            \multicolumn{1}{c}{$170.3_{34.2}$} &
            \multicolumn{1}{c}{$236.4_{36.8}$} &
            \multicolumn{1}{c}{$164.8_{17.1}$} &
            \multicolumn{1}{c}{$2.3_{1.2}$} &
            \multicolumn{1}{c}{$0.2_{0.4}$} &
            \multicolumn{1}{c}{$\bf0.1_{\bf0.3}$} \\
            & \multirow{1}{*}{SF2} &

            \multicolumn{1}{c}{$2.3_{1.3}$} &
            \multicolumn{1}{c}{$156.8_{21.2}$} &
            \multicolumn{1}{c}{$\bf0.0_{\bf0.0}$} &
            \multicolumn{1}{c}{$3.0_{1.41}$} &
            \multicolumn{1}{c}{$\bf0.0_{\bf0.0}$} &
            \multicolumn{1}{c}{$\bf0.0_{\bf0.0}$} \\
            & \multirow{1}{*}{SF5} &

            \multicolumn{1}{c}{$90.2_{34.5}$} &
            \multicolumn{1}{c}{$165.2_{22.0}$} &
            \multicolumn{1}{c}{$10.8_{6.1}$} &
            \multicolumn{1}{c}{$2.7_{0.9}$} &
            \multicolumn{1}{c}{$0.1_{0.3}$} &
            \multicolumn{1}{c}{$\bf0.0_{\bf0.0}$}
            \\\specialrule{.075em}{.0em}{.0em}

            \multirow{4}{*}{1000} &
            \multirow{1}{*}{ER2} &

            \multicolumn{1}{c}{-} &
            \multicolumn{1}{c}{-} &
            \multicolumn{1}{c}{-} &
            \multicolumn{1}{c}{-} &
            \multicolumn{1}{c}{$0.4_{0.5}$} &
            \multicolumn{1}{c}{$\bf0.1_{\bf0.3}$} \\
            & \multirow{1}{*}{ER5} &

            \multicolumn{1}{c}{-} &
            \multicolumn{1}{c}{-} &
            \multicolumn{1}{c}{-} &
            \multicolumn{1}{c}{-} &
            \multicolumn{1}{c}{$21.8_{3.8}$} &
            \multicolumn{1}{c}{$\bf8.9_{\bf4.2}$} \\
            & \multirow{1}{*}{SF2} &

            \multicolumn{1}{c}{-} &
            \multicolumn{1}{c}{-} &
            \multicolumn{1}{c}{-} &
            \multicolumn{1}{c}{-} &
            \multicolumn{1}{c}{$\bf0.0_{\bf0.0}$} &
            \multicolumn{1}{c}{$\bf0.0_{\bf0.0}$} \\
            & \multirow{1}{*}{SF5} &

            \multicolumn{1}{c}{-} &
            \multicolumn{1}{c}{-} &
            \multicolumn{1}{c}{-} &
            \multicolumn{1}{c}{-} &
            \multicolumn{1}{c}{$0.3_{0.5}$} &
            \multicolumn{1}{c}{$\bf0.0_{\bf0.0}$}

            \\\hline\hline
        \end{tabular}
 \vspace{0.1in}
\end{table*}

Table~\ref{tab:baseline_results} presents the structural Hamming distance (SHD) of baseline methods and our method on 3000 samples ($n=3000$). Nodes number $p$ is noted in the first column. Graph type and edge level are noted in the second column.
We only report the SHD of different algorithms due to page limitations. And we find that other metrics such as true positive rate (TPR), false discovery rate (FDR), false positive rate (FPR), and F1 score have similar comparative performance with SHD. We also test bottom-up EQVAR, which is equivalent to LISTEN. The result is worse than top-down EQVAR (EV-TD) in this synthetic  experiment, so we do not include the result in the table. For $p=1000$ graphs, we only report the result of EV-TD and CDCF since other algorithms spend too much time (longer than a week).

We run our methods on 10 randomly generated graphs and report the mean and variance in the table.
Figure~\ref{fig:compare} plots the SHD results tested on 100 nodes graph recovering from different sample sizes. In the low dimension setting ($p < n$), we choose EV-TD, LISTEN (LTN), and LISTEN with Cholesky estimator (LTN-CH) as baselines. In the high dimension setting ($p > n$), we choose high dimension top-down (EV-HTD) and LISTEN (LTN) as baselines.
We can see that CDCF achieves significantly better performance compared with previous baselines.
In most cases, CDCF can reconstruct the DAG structure exactly  according to the observing data, while the other algorithms
fail to recover the ground truth graphs.

\newpage

\begin{figure*}[t]
\begin{center}
\includegraphics[width=6.7in]
{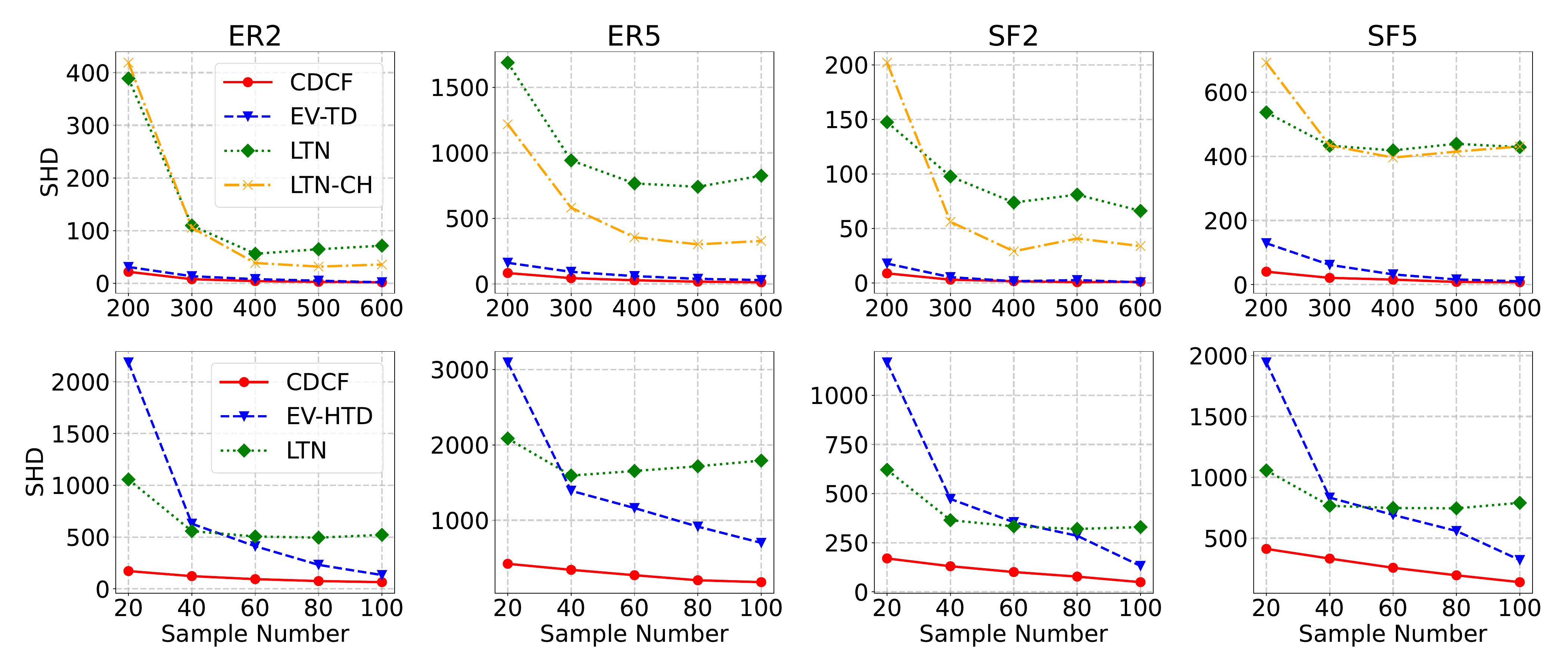}
\end{center}

\vspace{-0.2in}

\caption{Performance (SHD) tested on 100 nodes graph recovering from different sample numbers.} \label{fig:compare}
\end{figure*}

\begin{table*}[!h]
\begin{center}
\caption{Running time (seconds) on 30 and 100 nodes over 3000 samples. \label{tab:result_time}}
\setlength{\tabcolsep}{0.5mm}{
\begin{tabular}{l rrrr rrrr}
\hline\hline
& \multicolumn{4}{c}{\bf30} & \multicolumn{4}{c}{\bf100} \cr
\cline{2-5} \cline{6-9}
& ER2 & ER5 & SF2  & SF5 & ER2 & ER5 & SF2 & SF5 \\  \hline

CDCF  & $\bf0.004$ & $ \bf0.005$ & $\bf0.004$&$\bf 0.005$ & $ \bf0.017$& $ \bf0.016 $&$ \bf0.016$  & $\bf 0.017$ \\
EV-TD & $0.19$ & $ 0.16$ & $0.12$&$ 0.12$ & $ 14.42 $& $ 12.88 $&$ 15.04$  & $ 14.78$ \\
LISTEN &$0.26$ & $0.13$& $0.13$&$0.14$ & $13.97$& $13.41$&$13.42$&$15.43$\\
EV-HTD & $8.27$ & $ 7.48$ & $6.72$&$ 12.50$ & $ 260.74 $& $ 302.36 $&$ 241.59$  & $  387.92$ \\
DAG-GNN &$49.15$&$49.02$& $38.44$ & $41.03$ & $137.25$&$238.71$& $158.13$&$187.21$\cr
NPVAR & $84.24$ & $82.57$ &$108.37$&$109.13$&$9867.96$&$9084.78$&$10667.88$&$10173.89$\\
NOTEARS  & $78.19$ & $ 597.16$ & $51.57$&$ 306.31$ & $ 3237.8 $& $ 1803.30 $&$ 880.19$  & $ 4159.82$ \\
CORL1  & $17573.08$ & $ 18799.21$ & $16422.11$ & $16588.30$ & $ - $& $ -$&$-$  & $-$ \\ \hline\hline
\end{tabular}
}
\end{center}\vspace{-0.1in}
\end{table*}

Table~\ref{tab:result_time} shows the running time, tested on a 2.3 GHz single Intel Core i5 CPU. 
As illustrated in the table, CDCF is dozens or hundreds of times faster than EV-TD and LISTEN and tens of thousands of times faster than NOTEARS.

\subsubsection{Proteins Dataset}
We consider a bioinformatics dataset~\citep{sachs2005causal}
consisting of continuous measurements of expression levels of proteins and phospholipids in the human immune
system cells. This is a widely used dataset for research on graphical models, with experimental annotations accepted by the biological research community. Following the previous algorithms setting, we noticed that different previous papers adopted different observations. To included them all, we considered the observational 853 samples from the ``CD3, CD28'' simulation tested by~\citet{teyssier2005ordering,lachapelle2020gradient,zhu2020causal} and all 7466 samples from nine different simulations tested by~\citet{zheng2018dags, zheng2020learning, ng2019graph}.

\newpage

\begin{table}[h]
\begin{center}
\caption{Results on Proteins datasets.\label{tab:result_proteins_853}}\vspace{0.1in}
\setlength{\tabcolsep}{0.5mm}{
\begin{tabular}{c|c|c|c|c|c|c|c|c}
\hline\hline
\bf Datasets & \bf Methods & \bf FDR & \bf TPR & \bf FPR &\bf SHD & \bf N & \bf P & \bf F1 \\  \hline
\multirow{8}{*}{\makecell{853 samples \\ 17 edges}}
 &CDCF-V     & $0.533$&$\bf0.412$&$0.210$&$11$&$15$&$0.467$&$0.438$ \\
 &CDCF-S    & $\bf0.500$&$\bf0.412$&$\bf0.184$&$\bf10$&$14$&$\bf0.500$&$\bf0.452$  \\
 &CDCF-VS  & $\bf0.500$&$\bf0.412$&$\bf0.184$&$\bf10$&$14$&$\bf0.500$&$\bf0.452$ \\

 &NOTEARS & $0.588$ & $\bf0.412 $ &$0.263$& $13$ & $17$ & $0.412$  &$0.412$ \\
 &NOTEARSMLP & $0.733$ & $0.235 $ &$0.290$& $18$ & $15$ & $0.267$  &$0.250$ \\
 &CORL1\&2& $0.533$ & $\bf0.412$ &$0.211$& $11$ & $15$ &$0.467$ & $0.438$ \\
 &EV-TD&$0.645$&$0.294$&$0.237$&$17$&$14$&$0.357$&$0.323$\\
 &LISTEN&$0.750$&$0.176$&$0.237$&$18$&$12$&$0.250$&$0.207$\\
 &NPVAR&$0.800$&$0.176$&$0.316$&$19$&$15$&$0.200$&$0.188$\\
 &DAG-GNN &$0.588$&$\bf0.412$&$0.263$&$15$&$17$&$0.412$&$0.412$ \\
 \hline
\multirow{7}{*}{\makecell{7466 samples \\ 20 edges}}
 &CDCF-V     & $0.667$&$0.400$&$0.457$&$21$&$24$&$0.333$&$0.364$ \\
 &CDCF-S     & $0.611$&$0.350$&$0.314$&$17$&$18$&$0.389$&$0.368$ \\
 &CDCF-VS    & $\bf0.556$&$0.400$&$\bf0.286$&$\bf16$&$18$&$\bf0.444$&$\bf0.421$ \\

 &NOTEARS & $0.650$ & $0.350 $ &$0.371$& $20$ & $20$& $0.350$ &$0.350$ \\
 &NOTEARSMLP & $0.800$ & $0.200 $ &$0.457$& $26$ & $20$ & $0.200$&$0.200$ \\
 &CORL1\&2& $0.667$ & $0.400$ &$0.457$& $21$ & $24$& $0.333$ & $0.363$ \\
&EV-TD&$0.700$&$0.300$&$0.400$&$25$&$20$&$0.300$&$0.300$\\
&LISTEN&$0.714$&$0.300$&$0.429$&$23$&$21$&$0.286$&$0.293$\\
&NPVAR&$0.679$&$\bf0.450$&$0.543$&$24$&$28$&$0.321$&$0.375$\\
&DAG-GNN&$0.650$&$0.350$&$0.371$&$20$&$20$&$0.350$&$0.350$ \\
 \hline\hline
\end{tabular}
}
\end{center}
\end{table}

We notice that CDCF determines the ordering by the variance \eqref{ik}.
One may also take the sparsity into consideration.
In fact, we make the three criterion below to introduce sparsity into CDCF.
We select $i_k$ according to one of the following criteria:

\begin{itemize}

\item[{\bf(V)}] $i_k=\argmin_{k\le j \le p} \alpha_j^2$.
Under the assumption that the noise variance of the child variable
is approximately larger than that of its parents, it is reasonable/natural to
select the index that has the lowest estimation of the noise variance.
This criterion is same with CDCF introduced in Section~\ref{sec:sem}.

\item[{\bf(S)}] $i_k=\argmin_{k\le j \le p} \|\wtd\mU_{k-1}^{-1}\bm\beta_j\|_1$.
When the adjacent matrix $\mT$ is sparse, and the noise variables are independent, we would like to select the index that leads to the most sparse column of $\wtd\mU_k^{-1}$.
This criterion is especially useful when the number of samples is small.

\item[{\bf(VS)}] $i_k =\argmin_{k\le j \le p} \|\wtd\mU_{k-1}^{-1}\bm\beta_j\|_1 \sqrt{\big|\alpha_j^2 - \frac{1}{k-1}\sum_{h=1}^{k-1}\frac{1}{[\wtd\mU_{k-1}^{-1}]_{hh}^2}\big|}$.
We empirically combine criterion {\bf (V)} and criterion {\bf (S)} together
to take both aspects (variance and sparsity) into account.
Numerically, we found that this criterion achieves the best performance in real-world data.
\end{itemize}

We report the experimental results on both settings in Table~\ref{tab:result_proteins_853}.
The evaluation metric is FDR, TPR, FPR, SHD, predicted nodes number (N), precision (P), F1 score. As the recall score equals TPR, we do not include it in the table. In both settings, CDCF-VS achieves state-of-the-art performance.\footnote{For  NOTEARS-MLP, Table~\ref{tab:result_proteins_853} reports the results reproduced by the code provided in~\citet{zheng2020learning}.}
Several reasons make the recovered graph not exactly the same as the expected one. The ground truth graph suggested by the paper is mixed with directed and indirect edges. Under the settings of SEM, the node ``PKA'' is quite similar to the leaf nodes since most of its edges are indirect while the ground truth graph notes it as the out edges. Non-linearity would not be an impact issue here since both NOTEARS and our algorithm  achieve decent results. In the meantime, we do not deny that further extension of our algorithm to non-linear representation would witness an improvement on this dataset.

\subsubsection{Knowledge Base Dataset}
We test our algorithm on FB15K-237 dataset~\citep{toutanova2015representing} in which the knowledge is organized as $\{Subject, Predicate, Object\}$ triplets. The dataset has 15K triplets and 237 types of predicates. In this experiment, we only consider the single jump predicate between the entities, which have 97 predicates remaining. We want to discover the causal relationships between the predicates. We organize the observation data as each sample corresponds to an entity with awareness of the position (Subject or Object), and each variable corresponds to a predicate in this knowledge base.

\begin{figure}[b!]
\vspace{-0.1in}
\centering
\includegraphics[width=3.1in,trim={6.6cm 5cm 1.5cm 5cm},clip]
{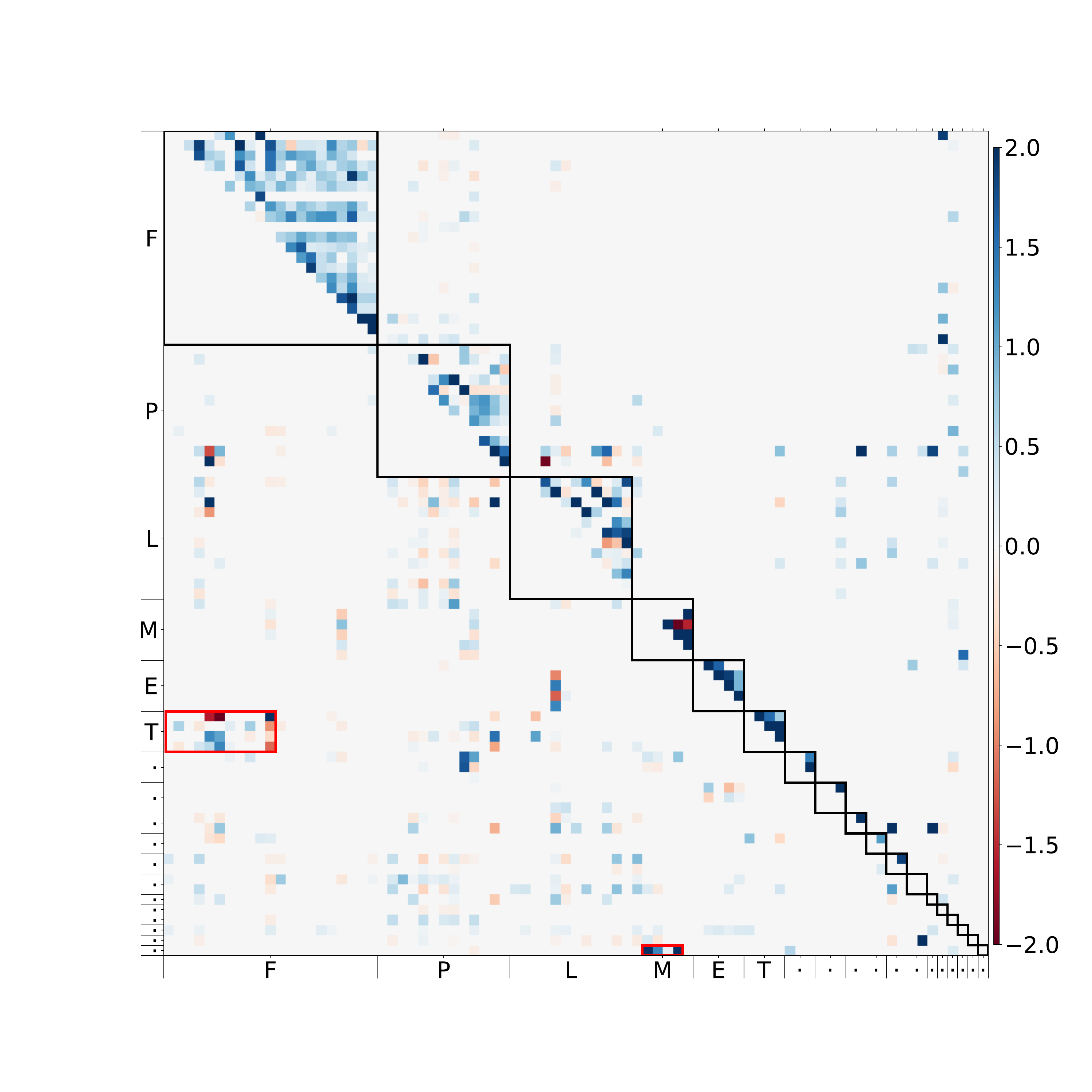}
\qquad
\small
\setlength{\tabcolsep}{0.5mm}
\begin{tabular}[b]{ll}\hline
Parent & Child \\ \hline
Medicine/Disease  & People/CauseOfDeath \\
Medicine/RiskFactors & People/CauseOfDeath \\
Broadcast/Artist     & Music/Artist\\
Broadcast/Artist     & Music/Instrumentalists\\
Tv/ProgramCreator&Tv/Languages\\
Tv/ProgramCreator&Tv/OriginCountry\\
Tv/Languages&Tv/OriginCountry\\
Tv/Genre&Film/Country \\
Film/Director&Media/NetflixGenre\\
Film/StoryBy&Film/Prequel \\
Film/WrittenBy&Film/Genre\\
Disease/RiskFactors&Disease/symptom\\
Sports/Colors&Sports/Teams\\
Person/Nationality&Person/PlaceOfBirth\\
Location/TimeZones&Location/Country\\
Education/Campuses&Education/SchoolType\\
Olympics/Countries&Event/Locations \\
Olympics/Countries&People/Language \\
\hline
\end{tabular}
\caption{The recovered weighted adjacent matrix (left) and examples of the high confidence relation pairs (right) on FB15k-237 dataset.}\label{fig:fb}
\vspace{-0.1in}
\end{figure}

In Figure~\ref{fig:fb}, we give the adjacent weighted matrix of the generated graph and several examples with high confidence (larger than 0.5). In the left figure, the axis label notes the first capital letter of the domain of the relations. Some of them are replaced with a dot to save space.
Specifically, the axis labels of Figure~\ref{fig:fb} are `Film', `People', `Location', `Music', `Education', `Tv', `Medicine', `Sports', `Olympics', `Award', `Time', `Organization', `Language', `MediaCommon', `Influence', `Dataworld', `Business', `Broadcast' from left to right for x-axis and top to bottom for y-axis, respectively. The adjacent matrix plotted in the Figure is re-permuted to make the relations in the same domain close to each other. We keep the adjacent matrix inside a domain an upper triangular matrix. Such typology is equivalent to the generated matrix with the original order. The domain clusters are denoted in black boxes at the diagonal of the adjacent matrix. The red boxes denoted the cross-domain relations that are worth paying attention to. Consistent with the innateness of human sense, the recovered relationships inside a domain are denser than those across domains. In the cross-domain relations, we found that the predicate in domain ``TV'' (``T'') has many relations with the domain ``Film'' (``F''), the domain ``Broadcast'' (last row) have many relations with the domain ``Music'' (``M''). Several cases of the predicted causal relationships are listed on the right side of Figure~\ref{fig:fb}, we can see that the discovered indication relations between predicates are quite reasonable.


\vspace{-0.05in}
\subsection{Experiments for CDCF$^+$}
We compare  CDCF$^+$ with BIC guided search (B-S) algorithm mentioned in~\citet{adams2021identification}. Their results show 22 different identifiable (up to re-indexing of the latent variables) typologies with three observable nodes and 0-2 unobservable nodes, which are shown in Figure~\ref{fig:graphtypes}. Following their  settings, we test Algorithm~\ref{alg2} on the 22 graphs.
We randomly generate 10 datasets for each graph and provide the average results in Table~\ref{tab:result_latent}. For each random dataset, the result is recorded as the best SHD score among the latent variable permutations.

\begin{figure*}[t]
\centering
\includegraphics[width=6.5in]
{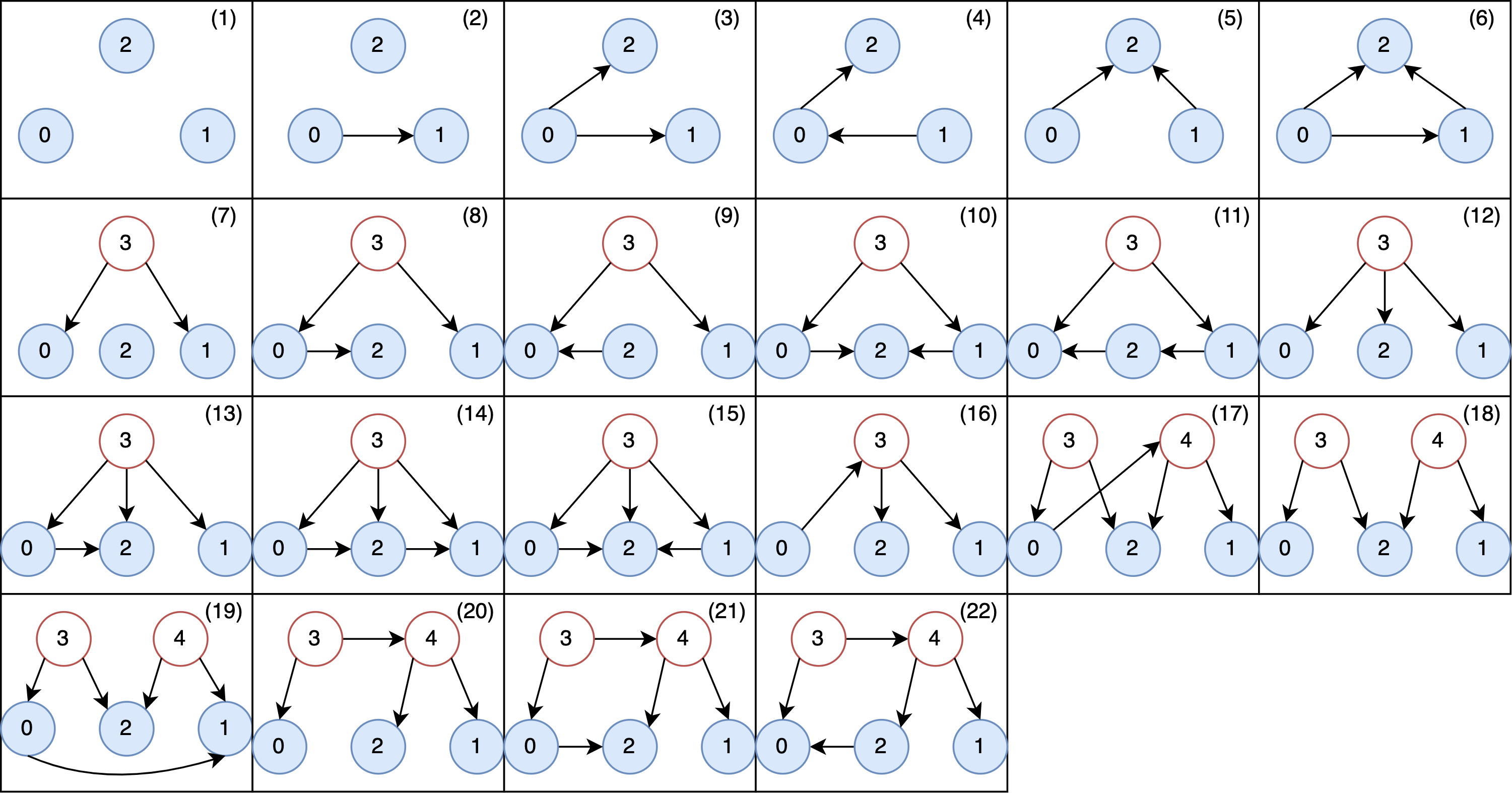}
\caption{Illustration of 22 identifiable graph types. The solid nodes with numbers 0, 1, 2 represent the observed variables, and the hollow nodes with numbers 3, 4 represent the latent variables.}\label{fig:graphtypes}\vspace{0.05in}
\end{figure*}

\begin{table*}[!h]
\begin{center}
\caption{Average SHD  on partially observable models.    \label{tab:result_latent}\vspace{-0.1in}}
\setlength{\tabcolsep}{0.45mm}{\small
\begin{tabular}{c|c|c|c|c|c|c|c|c|c|c|c|c|c|c|c|c|c|c|c|c|c|c|c}
\hline\hline
\multirow{2}{*}{\bf Sample}
& \multirow{2}{*}{\bf Method}
& \multicolumn{22}{c}{\bf Graph Type} \cr
\cline{3-24}
&& 1 & 2 & 3  & 4 & 5 & 6 & 7 & 8 & 9 & 10 & 11&12&13&14&15&16&17&18&19&20&21&22 \\  \hline
\multirow{2}{*}{5e3}
&B-S
&$\bf0.0$&$1.6$&$1.0$&$2.2$&$\bf0.0$
&$3.8$&$3.0$&$4.4$&$3.0$&$4.6$
&$4.5$&$3.5$&$4.6$&$3.8$&$6.6$
&$4.6$&$\bf5.6$&$6.0$&$6.1$&$5.4$
&$5.1$&$5.7$\cr
&\textsc{our}
&$\bf0.0$&$\bf0.0$&$\bf0.0$&$\bf0.0$&$\bf0.0$
&$\bf0.0$&$\bf0.0$&$\bf0.0$&$\bf0.0$&$\bf0.0$
&$\bf0.0$&$\bf0.0$&$\bf0.0$&$\bf0.0$&$\bf0.0$
&$\bf0.0$&$10.0$&$\bf0.0$&$\bf4.8$&$\bf0.0$
&$\bf1.3$&$\bf1.8$\\
\hline \hline
\multirow{2}{*}{1e4}
&B-S&$0.1$&$1.2$&$1.5$&$1.2$&$\bf0.0$&$4.4$&$3.0$&$4.0$&$3.0$&$4.8$&$4.7$&$5.0$&$5.0$&$4.8$&$6.2$&$4.3$&$\bf6.6$&$6.0$&$6.0$&$5.5$&$5.3$&$4.2$\cr
&\textsc{our}
&$\bf0.0$&$\bf0.0$&$\bf0.0$&$\bf0.0$&$\bf0.0$
&$\bf0.0$&$\bf0.0$&$\bf0.0$&$\bf0.0$&$\bf0.0$
&$\bf0.0$&$\bf0.0$&$\bf0.0$&$\bf0.0$&$\bf0.0$
&$\bf0.0$&$10.0$&$\bf0.0$&$\bf5.0$&$\bf0.0$
&$\bf1.2$&$\bf1.9$ \\
 \hline\hline
\end{tabular}
}
\end{center}
\end{table*}

We compare CDCF$^+$ with B-S~\citep{adams2021identification}, which is the most recent baseline and is applicable to equal variance and Gaussian noise cases without specific topology conditions. Algorithms like  FCI~\citep{spirtes1999algorithm}, RFCI~\citep{colombo2012learning}, and ICD~\citep{rohekar2021iterative} require specific topology conditions and can only recover the structure up to a partial ancestral graph, which is an equivalence class of the true causal graph, even in the situation that the causal graph is identifiable which is the setting of this experiment. Algorithms like FOFC~\citep{kummerfeld2016causal}, BPC~\citep{silva2006learning}, can only find the causal
clusters, i.e., locating latent variables, under the cases when the latent variable have at least three pure measurement variables. Besides, they can not discover the causal order of latent variables.

Table~\ref{tab:result_latent} shows that
our algorithm achieves significantly better performance and does not depend on the specified latent structures (e.g., confounder, mediator).
More experimental details are provided in  Appendix.

\section{Conclusion}\label{sec:conclusion}

We have proposed
two algorithms to tackle the DAG recovery problems with full/partial observations, respectively.
The first algorithm CDCF has theoretical guarantees for exact recovery and is better than
the existing methods in both time and sample complexities.
Experimental results on synthetic datasets and real-world data
demonstrate the efficiency and effectiveness of CDCF.
The second algorithm CDCF$^+$ is able to reveal latent variables and return a sparse DAG.
Numerical results indicate that CDCF$^+$  recovers most of the topology structures exactly,
and the performance is significantly better than previous algorithms.
Under what conditions (the ``restrictive'' assumption {\bf A3} is insufficient) Algorithm~\ref{alg2} is able to produce the true DAG requires a further investigation.


\bibliography{refs_scholar}
\bibliographystyle{plainnat}

\appendix

\section{Proofs}\label{adx:theory}
Our proofs for Theorems~2.1 and 2.2 are based on the perturbation analysis for the Cholesky factorization
of the covariance matrix, the following lemma plays the central role.

\begin{lemma}\label{lem:ll}
Let $\rvx\in\R^p$ be a zero-mean random vector, $\mSigma_{xx}=\E(\rvx \rvx^{\T})\in\R^{p\times p}$ be the covariance matrix, $\wtd\mSigma_{xx}$ be an estimator for $\mSigma_{xx}$ satisfying that
\begin{equation*}
\max_i|[\mSigma_{xx}]_{ii}-[\wtd\mSigma_{xx}]_{ii}\|\le \epsilon_d,\qquad
\|\mSigma_{xx}-\wtd\mSigma_{xx}\|\le \epsilon_2,
\end{equation*}
for some $\epsilon_d$, $\epsilon_2>0$.
Let the Cholesky factorizations of $\mSigma_{xx}$ and $\wtd\mSigma_{xx}$ be
$\mSigma_{xx}=\mU^{\T}\mU$ and $\wtd\mSigma_{xx}=\wtd\mU^{\T} \wtd\mU$, respectively,
where $\mU$ and $\wtd\mU$ are both upper triangular.
Then
\begin{align}
&|\|\mU_{:,i}\|^2 - \|\wtd\mU_{:,i}\|^2|\le \epsilon_d,
 &&\mbox{for $1\le i\le p$};
\label{ll1}\\
&|[\mU^{-1}]_{ij} - [\wtd\mU^{-1}]_{ij}|
\le \|\wtd\mU^{-1}\|_{2,\infty}\|\mU^{-\T}\|_{2,\infty} \sqrt{2\epsilon_2},
&&\mbox{for $i>j$}.\label{ll2}
\end{align}
\end{lemma}
\begin{proof}
For all $1\le i \le p$, we have
\begin{equation}\label{eq:cc}
|\|\mU_{:,i}\|^2 - \|\wtd\mU_{:,i}\|^2|
=|[\mSigma_{xx}]_{ii}-[\wtd\mSigma_{xx}]_{ii}|
\le \epsilon_d,
\end{equation}
which completes the proof for \eqref{ll1}.

Next, we show \eqref{ll2}.
Let
\begin{equation}\label{eq:ef}
\wtd\mU \mU^{-1} = \mI+\mF,\qquad
(\mI + \mF)^{\T}(\mI +\mF) = \mI+\mE,
\end{equation}
where $\mF$ is upper triangular.
We know that
\begin{align}
&\mU^{-1} - \wtd\mU^{-1}
= \wtd\mU^{-1}\mF,\label{eq:ll}\\
&\mE = \mU^{-\T} \wtd\mU^{\T} \wtd\mU \mU^{-1} - \mI
= \mU^{-\T}(\wtd\mSigma_{xx} - \mSigma_{xx}) \mU^{-1}.\label{eq:e}
\end{align}
Then it follows from \eqref{eq:ll} that for any $i>j$ it holds that
\begin{equation}\label{llmax}
|[\mU^{-1}]_{ij} - [\wtd\mU^{-1}]_{ij}|
\le \|\wtd\mU^{-1}\|_{2,\infty}\|\mF_{1:j,j}\|.
\end{equation}

It follows from the second equality of \eqref{eq:ef} that
\begin{equation}\label{2f}
 (1+\mF_{jj})^2 + \|\mF_{1:j-1,j}\|^2 = 1+\mE_{jj}.
\end{equation}
Therefore,
\begin{align}
|\mF_{jj}|&\le \sqrt{1+\mE_{jj}}-1
\stackrel{(a)}{\le} \sqrt{1+\|\mU^{-\T}\|_{2,\infty}^2 \epsilon_2} - 1
\le \frac12 \|\mU^{-\T}\|_{2,\infty}^2 \epsilon_2,\label{fjj}
\\
\mF_{jj}^2+\|\mF_{1:j-1,j}\|^2
&\stackrel{(b)}{\le}
 2|\mF_{jj}| + \|\mU^{-\T}\|_{2,\infty}^2 \|\wtd\mSigma_{xx}-\mSigma_{xx}\|
\stackrel{(c)}{\le} 2 \|\mU^{-\T}\|_{2,\infty}^2 \epsilon_2, \label{f2inf}
\end{align}
where (a) uses \eqref{eq:e} and $\|\mSigma_{xx}-\wtd\mSigma_{xx}\|_2\le \epsilon$,
(b) uses \eqref{2f},
(c) uses \eqref{fjj} and \eqref{eq:e} and $\|\mSigma_{xx}-\wtd\mSigma_{xx}\|\le \epsilon_2$.
Substituting \eqref{f2inf} into \eqref{llmax}, we obtain
\begin{align*}
|[\mU^{-1}]_{ij} - [\wtd\mU^{-1}]_{ij}|
&\le \|\wtd\mU^{-1}\|_{2,\infty}\sqrt{\mF_{jj}^2+\|\mF_{1:j-1,j}\|^2}
\le \|\wtd\mU^{-1}\|_{2,\infty}\|\mU^{-\T}\|_{2,\infty} \sqrt{2\epsilon_2}.
\end{align*}
This completes the proof.
\end{proof}

\vspace{0.1in}
\noindent{\bf Theorem 2.1}\;
For the linear SEM model \eqref{model},
assume \eqref{wptp}, {\bf A1} and {\bf A2}.
Let $\wtd\mSigma_{xx}$ be an estimator for $\mSigma_{xx}=[\E(\emX_j\emX_k)]$ and
$\max_i|[\mSigma_{xx}]_{ii}-[\wtd\mSigma_{xx}]_{ii}|\le \epsilon_d$
for some $\epsilon_d>0$.
If
$\epsilon_d < \frac{\Delta}{4}$,
Algorithm~\ref{alg1} recovers the ordering exactly.

\begin{proof}
$\wtd{\underline{\vi}}=[\wtd i_1,\dots,\wtd i_p]=[1,\dots,p]\wtd\mP$,
where 
$\wtd\mP$ is the output of Algorithm~\ref{alg1}.
Denote $
\wtd{\underline{\vi}}_k=[\wtd i_1,\dots,\wtd i_k]$,
$\mU=\diag(\sigma_{i_1},\dots,\sigma_{i_p})(\mI-\mT)^{-1}$,
$\vu_k=\mU_{1:k-1,k}$,
$\wtd\vu_k = [\wtd\mU_p]_{1:k-1,k}$.

\vspace{0.1in}
In Algorithm~\ref{alg1},
we have
\begin{equation}\label{mc}
[\wtd\mSigma_{xx}]_{\wtd{\underline{\vi}},\wtd{\underline{\vi}}}
= \wtd\mU_p^{\T}\wtd\mU_p.
\end{equation}
Consider the $k$th diagonal entries of $[\wtd\mSigma_{xx}]_{\wtd{\underline{\vi}},\wtd{\underline{\vi}}}$.
By calculations, we get
\begin{align}
[\wtd\mSigma_{xx}]_{\wtd i_k,\wtd i_k}
= [\wtd\mU_p]_{kk}^2  +  \|\wtd\vu_k\|^2. \label{ckk2}
\end{align}
Recall \eqref{chol}, we have
\begin{equation}\label{sigxxstar}
[\mSigma_{xx}]_{i_k,i_k} = \sigma_{i_k}^2 +\|\vu_k\|^2.
\end{equation}

Next, we show that all root nodes can be found by Algorithm~\ref{alg1}.
Let $i=i_s$ be a root node, $j=i_t$ be not, by calculations, we have
\begin{equation*}
[\wtd\mSigma_{xx}]_{jj}
\stackrel{(a)}{\ge} [\mSigma_{xx}]_{jj} - \epsilon_d
\stackrel{(b)}{=} \sigma_{i_t}^2 + \|\vu_t\|^2  - \epsilon_d
\stackrel{(c)}{\ge} [\mSigma_{xx}]_{ii}
+\sigma_{i_t}^2 - \sigma_{i_s}^2+\|\vu_t\|^2 - \epsilon_d
\stackrel{(d)}{\ge} [\wtd\mSigma_{xx}]_{ii} + \Delta -2\epsilon_d
\stackrel{(e)}{>} [\wtd\mSigma_{xx}]_{ii},
\end{equation*}
where (a) uses Lemma~\ref{lem:ll},
(b), (c) use \eqref{sigxxstar},
(d) uses Lemma~\ref{lem:ll}, {\bf A2},
(e) is due to $\epsilon<\frac12 \Delta$.
So we have
$\{\wtd i_1,\dots,\wtd i_{p_1}\} = \{i_1,\dots,i_{p_1}\}$, i.e.,
all root nodes can be found.

\newpage

Now assume that Algorithm~\ref{alg1} is able to find the nodes in the first $\ell$ ($1\le \ell <r$) layers, layer by layer, specifically,
$\{\wtd i_1,\dots,\wtd i_{p_1}\} = \{i_1,\dots,i_{p_1}\}$, $\ldots$,
$\{\wtd i_{p_1+\cdots+p_{\ell-1}+1},\dots,\wtd i_{p_1+\cdots+p_{\ell}}\} = \{i_{p_1+\cdots+p_{\ell-1}+1},\dots,i_{p_1+\cdots+p_{\ell}}\}$.
We show that all nodes in the $\ell+1$st layer
can also be found by Algorithm~\ref{alg1}.
Let $i=i_s$ be a node in the $\ell+1$st layer,
$j=i_t$ be not.
We have
\begin{align}
[\wtd\mSigma_{xx}]_{jj} - \|[\wtd\vu_t]_{1:j}\|^2
&\stackrel{(f)}{\ge} [\mSigma_{xx}]_{jj} - \|[\vu_t]_{1:j}\|^2 - 2\epsilon_d
\stackrel{(g)}{=} \sigma_{i_t}^2 +\|\vu_t\|^2 - \|[\vu_t]_{1:j}\|^2 - 2\epsilon_d\notag\\
&= \sigma_{i_t}^2 +\|[\vu_t]_{j+1:t-1}\|^2  - 2\epsilon_d
\stackrel{(h)}{\ge} \sigma_{i_s}^2 + \Delta   - 2\epsilon_d
\stackrel{(i)}{=}[\mSigma_{xx}]_{ii} - \|\vu_s\|^2 + \Delta - 2\epsilon_d\notag\\
&\stackrel{(j)}{\ge} [\wtd\mSigma_{xx}]_{ii} -\|\wtd\vu_s\|^2 + \Delta -4\epsilon_d
\stackrel{(k)}{>} [\wtd\mSigma_{xx}]_{ii} -\|\wtd\vu_s\|^2,\notag
\end{align}
where (f), (j) use Lemma~\ref{lem:ll},
(g), (i) use \eqref{sigxxstar},
(h) uses {\bf A2},
(k) is due to $\epsilon<\frac14 \Delta$.
So all nodes in the $\ell+1$st layer can be found.
By mathematical induction, we get the conclusion.
\end{proof}

\vspace{0.1in}

\noindent{\bf Theorem 2.2}\;
Follow the notations and assumptions in Theorem~\ref{the:recover}.
Denote
\begin{equation*}
\mu=\|\mI-\mT\|_{2,\infty}\|\mI-\mT^{\T}\|_{2,\infty},  \quad
\rho = \frac{\max_i\sigma_i}{\min_i\sigma_i^2}, \quad
\omega=\min_{\mT_{ij}\ne 0}|\mT_{ij}|
\end{equation*}
Assume
$\|\mSigma_{xx}-\wtd\mSigma_{xx}\|\le \epsilon_2$ for some $\epsilon_2>0$.
If
$\epsilon_2\lesssim\frac{\omega^2}{8\mu^2\rho^2}$,
then Algorithm~\ref{alg1}
 recovers the graph structure exactly.

\begin{proof}
Recall that
\begin{equation*}
\mP^{\T}\mSigma_{xx}\mP = \mU^{\T}\mU,\quad
\wtd\mP^{\T}\wtd\mSigma_{xx}\wtd\mP= \wtd\mU_p^{\T}\wtd\mU_p,
\end{equation*}
where $\mU= \mP^{\T}\mSigma_{nn}^{\frac12}\mP (\mI-\mT)^{-1}$.
Rewrite the first equality as
\begin{equation}\label{cholpsp}
\wtd\mP^{\T}\mSigma_{xx}\wtd\mP =  \wtd\mP^{\T}\mP \mP^{\T}\mSigma_{xx}\mP\mP^{\T}\wtd\mP
=(\bm\Pi^{\T} \mU^{\T}\bm\Pi)  (\bm\Pi^{\T}\mU \bm\Pi),
\end{equation}
where $\bm\Pi=\mP^{\T}\wtd\mP$.
By Lemma~\ref{lem:ll},   $\bm\Pi$ is block diagonal and corresponds with the within layer permutation.
Recall the structure of $\mT$ in \eqref{tblk},  we know that $\bm\Pi^{\T}\mU \bm\Pi$ is upper triangular,
which implies that \eqref{cholpsp} is a Cholesky factorization.
Then it follows from Lemma~\ref{lem:ll} that
\begin{equation*}
\|\wtd\mU_q^{-1} - (\bm\Pi^{\T}\mU \bm\Pi)^{-1}\|_{\max}
\le \|\wtd\mU_p^{-1} \|_{2,\infty} \|\mU^{-\T}\|_{2,\infty}\sqrt{2\epsilon_2}.
\end{equation*}
Since $\wtd\mU_p^{-1}\rightarrow \mU^{-1}$ as $\epsilon_2\rightarrow 0$,
we have $\mA:=  \Pi \diag(\wtd\mU_p^{-1})  \Pi^{\T} \rightarrow \diag(\sigma_{i_1}^{-1},\dots,\sigma_{i_p}^{-1})$, and
\begin{equation*}
\|\wtd\mU_q^{-1} - \bm\Pi^{\T}\mU^{-1} \bm\Pi \|_{\max}
\lesssim \|\mU^{-1} \|_{2,\infty} \|\mU^{-\T}\|_{2,\infty}\sqrt{2\epsilon_2}
\le \frac{1}{\min_i\sigma_i^2} \|\mI-\mT\|_{2,\infty} \|\mI-\mT^{\T}\|_{2,\infty}\sqrt{2\epsilon_2}
=\frac{\mu}{\min_i\sigma_i^2} \sqrt{2\epsilon_2}.
\end{equation*}
Then it follows that
\begin{equation*}
\|\bm\Pi \textsc{triu}(\wtd\mU_p^{-1}) \bm\Pi^{\T} - \mT \diag(\sigma_{i_1}^{-1},\dots,\sigma_{i_p}^{-1}) \|_{\max}
\lesssim
\frac{\mu}{\min_i\sigma_i^2} \sqrt{2\epsilon_2}.
\end{equation*}
And hence
\begin{equation*}
\|\bm\Pi \textsc{triu}(\wtd\mU_p^{-1}\mA) \bm\Pi^{\T} - \mT  \|_{\max}
\lesssim
\mu\rho\sqrt{2\epsilon_2}.
\end{equation*}

Let $\mu\rho \sqrt{2\epsilon_2}< \frac{\omega}{2}$,
we can recover the graph structure by truncating $\textsc{triu}(\wtd\mU_q^{-1}\mA)$.
The proof is completed.
\end{proof}

\noindent{\bf Proof of the sample complexity for exact ordering recovery}.
Denote $\gM=\max_i[\mSigma_{xx}]_{ii}$. Assume $\Delta=\gO(1)$.
For the random variable $\emX_i^2$, we have $0 \le \mathbb{E}(\emX_i^2) \le \gM<\infty$.
By Hoeffding's inequality, we get
\begin{equation*}
\mathbb{P}(|[\wtd\mSigma_{xx}]_{ii} - [\mSigma_{xx}]_{ii}|>t)\le 2\exp(-\frac{2nt^2}{\gM^2}).
\end{equation*}
Set $\epsilon=2\exp(-\frac{2nt^2}{\gM^2})$ and $t=\frac{\Delta}{4}$, we get
$n=\gO(\gM^2\log\frac{1}{\epsilon})$.
In other words, w.p. $\ge 1-\epsilon$, $|[\wtd\mSigma_{xx}]_{ii} - [\mSigma_{xx}]_{ii}|\le \frac{\Delta}{4}$ for $n=\gO(\gM^2\log\frac{1}{\epsilon})$.
Therefore, w.p. $\ge 1-\epsilon$, for all $1\le i\le p$, $|[\wtd\mSigma_{xx}]_{ii} - [\mSigma_{xx}]_{ii}|\le \frac{\Delta}{4}$ for $n=\gO(\gM^2\log\frac{p}{\epsilon})$.
The conclusion follows.

It is worth mentioning here that if we assume $\Delta=\gO(\gM)$ instead, the $\gM$ factor in the sample complexity can be removed.

\vspace{0.2in}
\noindent{\bf Proof of the sample complexity for exact graph structure recovery}.\  When the noise is sub-Gaussian, it holds with probability at least $1-2\exp(-ct^2)$ that
\begin{equation*}
\|\wtd\mSigma_{xx}-\mSigma_{xx}\|\le \|(\mI-\mT)^{-1}\|^2\max\{\delta,\delta^2\},
\end{equation*}
where $\delta=C\sqrt{\frac{p}{n}}+\frac{t}{\sqrt{n}}$, $c$ and $C$ are two constants.
For simplicity, let $\delta\ge 1$.
Set $2\exp(-ct^2)=\epsilon$, we get
\begin{equation*}
t=\gO(\sqrt{\log\frac{1}{\epsilon}}),\qquad
\delta = \gO(\sqrt{\frac{p}{n}}+\sqrt{\frac{1}{n}\log\frac{1}{\epsilon}}).
\end{equation*}
Then w.p. $\ge 1-\epsilon$, it holds
\begin{equation*}
\|\wtd\mSigma_{xx}-\mSigma_{xx}\|
\le \|(\mI-\mT)^{-1}\|^2\delta^2
\le \|(\mI-\mT)^{-1}\|^2\gO(\frac{p+\log\frac{1}{\epsilon}}{n}).
\end{equation*}
Let the right hand side be no more than $\frac{\omega^2}{8\mu^2\rho^2}$, we get
$n\ge\gO((p+\log\frac{1}{\epsilon}) \rho^2\mu^2 \|(\mI-\mT)^{-1}\|^2/\omega^2)$.
Sample complexities for other distributions of the noise can be  obtained similarly.

\vspace{0.2in}
\noindent{\bf Proof of Proposition 1}.\quad
Direct calculations give rise to
{\bf (a)} and {\bf (b)}.  Next,  we only show {\bf (c)}.

Using
$\mL^{\T}\mL=\mI+\vw\vw^{\T}$, we have
\begin{align*}
\prod_{i=1}^k \mL_{ii}^2=\det([\mL^{\T}\mL]_{1:k,1:k})
=\det(\mI+\vw_{1:k}\vw_{1:k}^{\T})
=1+\vw_{1:k}^{\T}\vw_{1:k}\footnotemark
=1+\|\vw_{1:k}\|^2.
\end{align*}
Note that, for any nonzero vector $\vx\in\R^p$, it holds $(\mI+\vx\vx^{\T})\vx=(1+\|\vx\|^2)\vx$,
$(\mI+\vx\vx^{\T})\vy=\vy$ for any $\vy^{\T}\vx=0$.
In other words, the eigenvalues of $\mI+\vx\vx^{\T}$ are $1+\|\vx\|^2$, $1,\dots,1$.
Therefore, $\det(\mI+\vx\vx^{\T})=\text{the product of all eigenvalues}=1+\|\vx\|^2$.
The conclusion follows immediately.

\vspace{0.2in}
\noindent{\bf Convergence of Algorithm~2.}\;
In Proposition~2, set the diagonal entries of $\wht\mU_{22}$ to $\wht\sigma$, then
one can construct
a unique vector $\vw$ and an upper triangular matrix $\mL$ with $\mV_{ii}=[\mU_{22}]_{ii}/\wht\sigma$ for all $i$ such that $\mL^{\T}\mL = \mI + \vw\vw^{\T}$.
Set $\wht\mU_{22}=\mL^{-1}\mU_{22}$, we know that the diagonal entries of $\wht\mU_{22}$ are all $\wht\sigma$.
Therefore, let $j$, $\underline{\vi}$ be the same as in Lines 7 and 10 of Algorithm~2, respectively,
we know that we can invert one row and column into $\wtd\mSigma_{xx}$ such that the diagonal entries of the Cholesky factor of $[\wtd\mSigma_{xx}]_{\underline{\vi},\underline{\vi}}$, denoted by $\alpha_1,\dots,\alpha_q$, satisfy
\begin{align}
&\frac{1}{\alpha_i}\ge\frac{1-\zeta}{\wht\sigma}, \quad \mbox{ for $i=1,\dots,j-1$};\qquad  \frac{1}{\alpha_i}=\frac{1}{\wht\sigma}> \frac{1-\zeta}{\wht\sigma}, \quad\mbox{ for $i=j,\dots,q$.}  \label{alpha}
\end{align}

By continuity of the optimization problem \eqref{opt} w.r.t. $\mu$, we take $\alpha_1,\dots,\alpha_q$ as functions of $\mu$.
For $\mu=0$, \eqref{opt} has a solution $\mS$ such that \eqref{alpha} holds.
So, for a sufficiently small $\mu$,
\eqref{opt} has a solution $\mS$ such that
\begin{align}
\frac{1}{\alpha_i}\ge\frac{1-\zeta}{\wht\sigma}, \quad \mbox{ for $i=1,\dots,j-1$};\qquad \frac{1}{\alpha_i}> \frac{1-\zeta}{\wht\sigma}, \quad\mbox{ for $i=j,\dots,q$.}  \label{alpha2}
\end{align}
This completes the proof.

\section{Additional Experiments}\label{adx:additional_experiments}

\subsection{Further Experiments on CDCF}\label{sec:diag_param}

\paragraph{Experiments with non-Gaussian Noise}
Figure~\ref{fig:compare_gumbel} and Figure~\ref{fig:compare_exp} provide the results of Gumbel and Exponential noises, respectively. As we can see from the result, our algorithm still performs better than the Eqvar method in different noise types.

\begin{figure}[ht]

\vspace{-0.1in}

\begin{center}
\includegraphics[width=6.5in]
{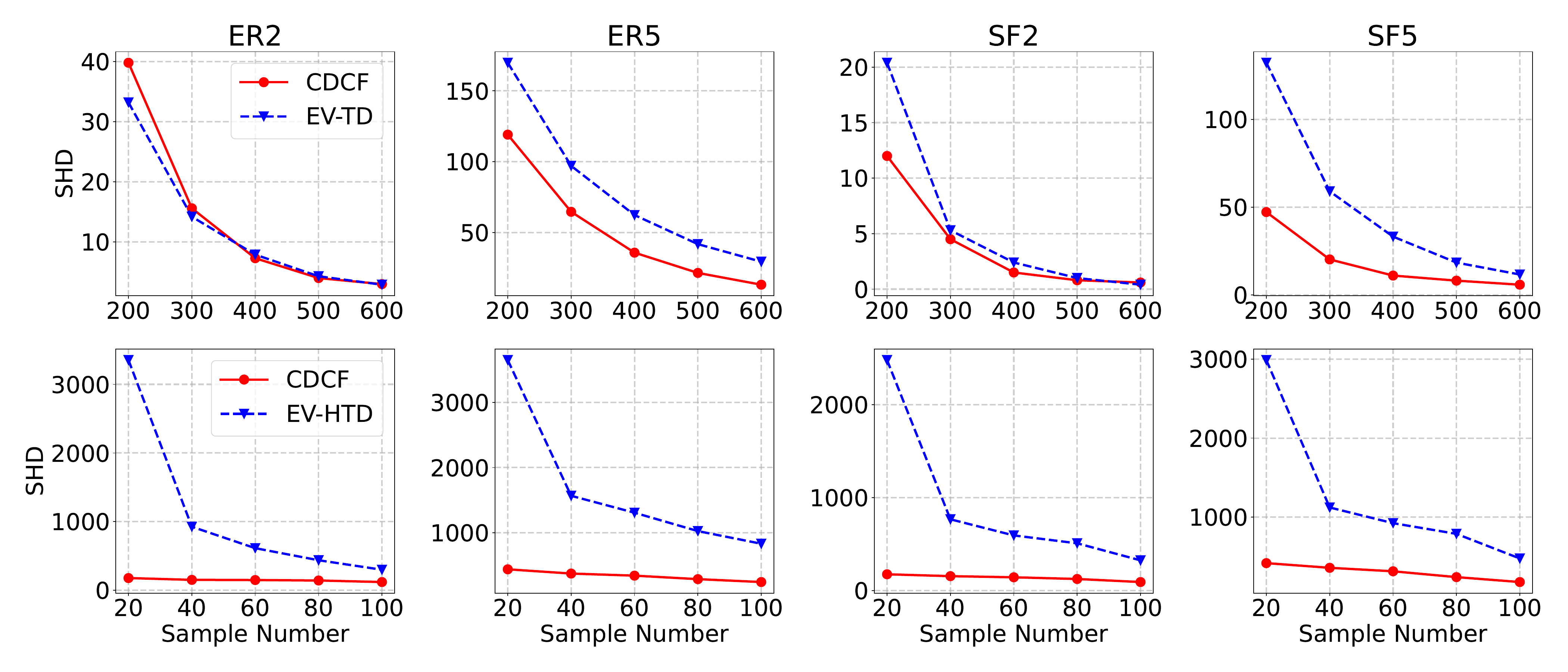}
\end{center}

 \vspace{-0.3in}

\caption{Performance (SHD)  on a 100-node graph, with the gumbel noise.} \label{fig:compare_gumbel}\vspace{-0.2in}
\end{figure}

\begin{figure}[h]
\begin{center}
\includegraphics[width=6.5in]
{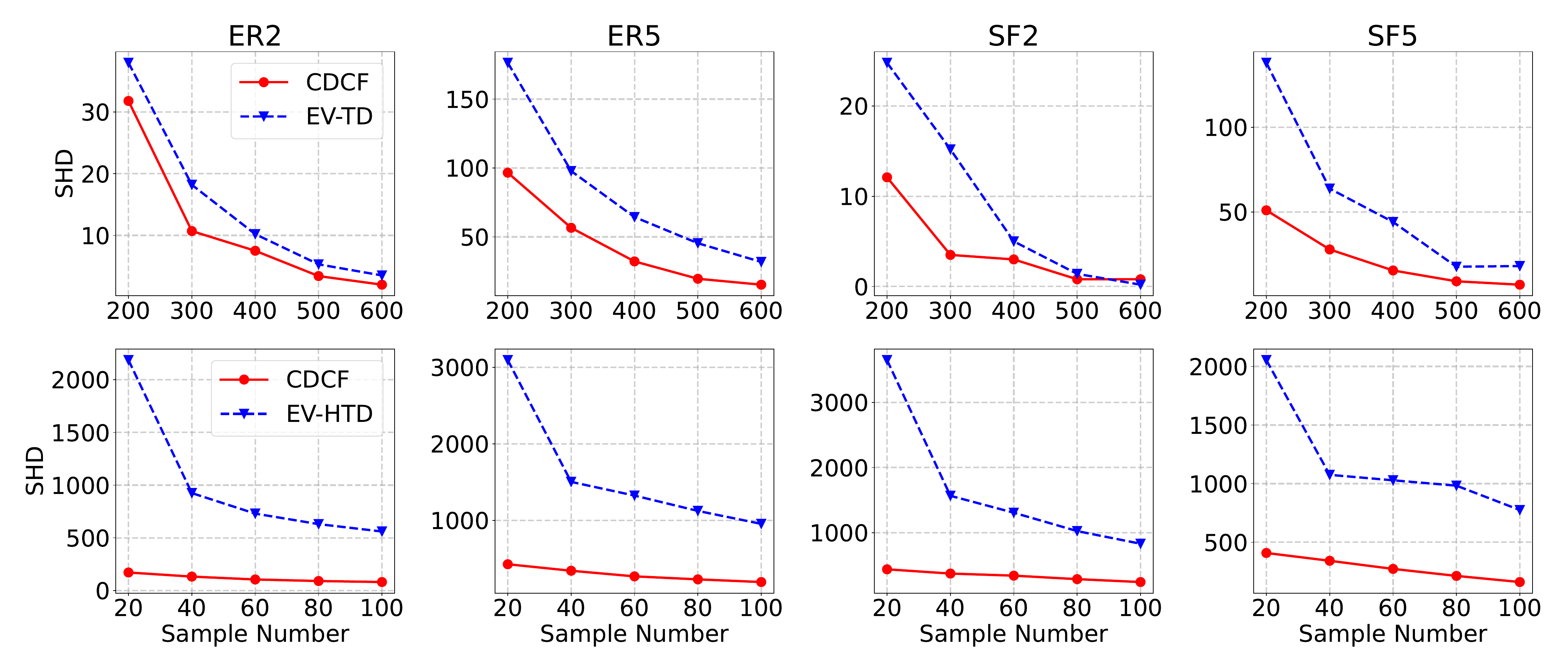}
\end{center}

 \vspace{-0.3in}

\caption{Performance (SHD) on 100 nodes graph, with the exponential noise.} \label{fig:compare_exp}\vspace{-0.3in}
\end{figure}

\begin{table}[h!]

\begin{center}
\caption{CDCF SHD results on with varying $\gamma$ and sample sizes on 100 nodes linear Gaussian SEM. \label{tab:lambda_vp_gauss}}\vspace{0.1in}
\setlength{\tabcolsep}{0.5mm}{
\begin{tabular}{c|c|c|c|c|c|c|c|c|c|c|c|c|c|c}
\hline\hline
& CDCF& \bf 200& \bf 300 & \bf 400 & \bf 500 &\bf 600 & \bf 700 & \bf 800 & \bf 900 & \bf 1000 & \bf 1500 & \bf 2000  & \bf 2500 & \bf 3000 \\  \hline
\multirow{10}{*}{ER2}
&$\gamma=0.0$&$114.3$&$32.2$&$11.7$&$4.6$&$2.5$&$1.4$&$\bf1.1$&$\bf0.5$&$\bf0.4$&$\bf0.3$&$\bf0.0$&$\bf0.0$&$\bf0.0$\cr

&$\gamma=1.0$&$21.9$&$7.9$&$4.3$&$\bf2.8$&$\bf2.0$&$\bf1.3$&$1.4$&$\bf0.5$&$\bf0.4$&$\bf0.3$&$\bf0.0$&$\bf0.0$&$\bf0.0$\cr

&$\gamma=2.0$&$\bf13.6$&$\bf7.2$&$\bf4.0$&$2.9$&$2.1$&$1.7$&$1.4$&$0.9$&$0.7$&$\bf0.3$&$\bf0.0$&$\bf0.0$&$\bf0.0$\cr

&$\gamma=3.0$&$14.8$&$8.1$&$4.5$&$3.5$&$2.7$&$2.0$&$1.8$&$1.4$&$0.7$&$\bf0.3$&$\bf0.0$&$\bf0.0$&$\bf0.0$\cr

&$\gamma=4.0$&$18.3$&$9.9$&$7.2$&$4.9$&$3.1$&$3.0$&$2.3$&$1.4$&$0.8$&$0.6$&$\bf0.0$&$\bf0.0$&$\bf0.0$\cr

&$\gamma=5.0$&$23.0$&$12.3$&$9.3$&$6.5$&$4.1$&$3.8$&$2.4$&$1.6$&$1.0$&$0.8$&$0.1$&$\bf0.0$&$\bf0.0$\cr

&$\gamma=6.0$&$27.0$&$14.9$&$11.4$&$7.7$&$4.5$&$4.4$&$3.6$&$1.9$&$1.7$&$0.8$&$0.1$&$0.1$&$\bf0.0$\cr

&$\gamma=7.0$&$33.0$&$19.1$&$13.7$&$9.1$&$6.0$&$4.9$&$4.3$&$2.5$&$2.4$&$0.9$&$0.2$&$0.1$&$\bf0.0$\cr

&$\gamma=8.0$&$37.3$&$22.1$&$15.5$&$10.2$&$7.0$&$5.7$&$5.1$&$2.9$&$3.1$&$1.2$&$0.2$&$0.1$&$\bf0.0$\cr

&$\gamma=9.0$&$42.5$&$25.1$&$17.6$&$11.6$&$8.6$&$6.9$&$5.8$&$3.6$&$3.4$&$1.2$&$0.2$&$0.1$&$0.1$\cr

&$\gamma=10.0$&$46.7$&$27.4$&$20.0$&$12.8$&$10.0$&$7.5$&$6.6$&$4.5$&$3.7$&$1.3$&$0.3$&$0.1$&$0.1$\\\hline

\multirow{10}{*}{SF2}
&$\gamma=0.0$&$37.1$&$7.2$&$1.9$&$2.6$&$\bf0.9$&$\bf0.3$&$\bf0.1$&$\bf0.1$&$0.2$&$\bf0.0$&$\bf0.0$&$\bf0.0$&$\bf0.0$\cr

&$\gamma=1.0$&$\bf8.8$&$\bf3.0$&$\bf1.6$&$\bf0.7$&$1.1$&$0.5$&$0.2$&$0.2$&$\bf0.1$&$\bf0.0$&$\bf0.0$&$\bf0.0$&$\bf0.0$\cr

&$\gamma=2.0$&$9.7$&$4.2$&$2.2$&$0.9$&$1.2$&$0.5$&$0.2$&$\bf0.1$&$\bf0.1$&$\bf0.0$&$\bf0.0$&$\bf0.0$&$\bf0.0$\cr

&$\gamma=3.0$&$11.5$&$6.2$&$2.9$&$1.4$&$1.2$&$0.4$&$0.3$&$\bf0.1$&$\bf0.1$&$\bf0.0$&$\bf0.0$&$\bf0.0$&$\bf0.0$\cr

&$\gamma=4.0$&$15.4$&$7.5$&$3.5$&$1.8$&$1.6$&$0.8$&$0.4$&$0.3$&$0.2$&$\bf0.0$&$\bf0.0$&$\bf0.0$&$\bf0.0$\cr

&$\gamma=5.0$&$18.1$&$8.9$&$4.3$&$2.7$&$2.1$&$1.3$&$0.5$&$0.5$&$0.4$&$\bf0.0$&$\bf0.0$&$\bf0.0$&$\bf0.0$\cr

&$\gamma=6.0$&$20.9$&$10.8$&$5.4$&$3.3$&$2.5$&$1.7$&$0.8$&$1.0$&$0.6$&$\bf0.0$&$\bf0.0$&$\bf0.0$&$\bf0.0$\cr

&$\gamma=7.0$&$23.5$&$13.1$&$6.7$&$5.2$&$4.9$&$2.7$&$1.5$&$1.3$&$0.7$&$\bf0.0$&$\bf0.0$&$\bf0.0$&$\bf0.0$\cr

&$\gamma=8.0$&$27.0$&$14.8$&$8.1$&$5.8$&$6.0$&$2.9$&$2.2$&$1.5$&$0.7$&$\bf0.0$&$\bf0.0$&$\bf0.0$&$\bf0.0$\cr

&$\gamma=9.0$&$29.4$&$16.8$&$9.8$&$6.7$&$6.7$&$3.5$&$2.7$&$2.0$&$1.1$&$\bf0.0$&$\bf0.0$&$\bf0.0$&$\bf0.0$\cr

&$\gamma=10.0$&$32.0$&$19.0$&$11.5$&$7.4$&$7.7$&$4.1$&$3.0$&$2.3$&$1.5$&$0.1$&$\bf0.0$&$\bf0.0$&$\bf0.0$\\\hline

\multirow{10}{*}{ER5}
&$\gamma=0.0$&$368.7$&$139.6$&$73.8$&$33.2$&$21.5$&$14.0$&$9.5$&$7.0$&$5.8$&$1.5$&$\bf0.6$&$\bf0.0$&$\bf0.0$\cr

&$\gamma=1.0$&$83.0$&$44.2$&$27.8$&$\bf17.5$&$\bf12.5$&$\bf9.2$&$\bf7.0$&$\bf5.1$&$\bf4.7$&$\bf1.3$&$0.7$&$\bf0.0$&$0.1$\cr

&$\gamma=2.0$&$\bf74.1$&$\bf41.7$&$\bf26.2$&$19.8$&$13.4$&$10.7$&$8.4$&$6.6$&$5.7$&$\bf1.3$&$1.1$&$\bf0.0$&$0.3$\cr

&$\gamma=3.0$&$82.1$&$52.3$&$33.1$&$23.0$&$18.2$&$14.5$&$11.9$&$8.0$&$7.2$&$2.4$&$1.6$&$0.2$&$0.4$\cr

&$\gamma=4.0$&$94.8$&$62.1$&$40.1$&$28.6$&$23.9$&$18.1$&$14.3$&$10.3$&$10.3$&$3.9$&$2.1$&$0.4$&$0.3$\cr

&$\gamma=5.0$&$109.7$&$72.3$&$51.4$&$37.6$&$30.1$&$22.6$&$17.6$&$14.1$&$12.8$&$5.6$&$2.6$&$0.8$&$0.4$\cr

&$\gamma=6.0$&$123.3$&$84.2$&$57.9$&$47.8$&$34.7$&$26.3$&$21.2$&$17.4$&$15.4$&$6.7$&$3.2$&$0.9$&$0.5$\cr

&$\gamma=7.0$&$140.3$&$100.2$&$67.9$&$53.1$&$39.9$&$31.2$&$25.0$&$20.2$&$18.2$&$7.8$&$3.7$&$1.0$&$0.7$\cr

&$\gamma=8.0$&$152.7$&$113.0$&$74.8$&$61.5$&$47.5$&$36.4$&$31.0$&$23.5$&$22.3$&$9.4$&$4.4$&$1.4$&$0.9$\cr

&$\gamma=9.0$&$162.5$&$122.0$&$83.1$&$67.7$&$51.6$&$42.3$&$34.6$&$28.5$&$25.7$&$10.7$&$4.9$&$1.9$&$1.4$\cr

&$\gamma=10.0$&$175.5$&$130.2$&$92.6$&$73.8$&$59.0$&$48.8$&$37.7$&$33.2$&$28.7$&$12.5$&$5.7$&$2.4$&$1.8$\\\hline

\multirow{10}{*}{SF5}
&$\gamma=0.0$&$92.6$&$31.9$&$\bf12.6$&$\bf6.0$&$\bf4.9$&$\bf3.9$&$\bf2.5$&$\bf1.8$&$\bf1.6$&$\bf0.3$&$\bf0.4$&$\bf0.0$&$\bf0.0$\cr

&$\gamma=1.0$&$\bf40.3$&$\bf21.0$&$15.6$&$8.3$&$7.2$&$5.3$&$4.1$&$2.3$&$2.8$&$0.5$&$\bf0.4$&$\bf0.0$&$\bf0.0$\cr

&$\gamma=2.0$&$57.2$&$35.7$&$23.3$&$16.1$&$11.2$&$9.6$&$7.5$&$5.9$&$5.1$&$1.7$&$0.6$&$\bf0.0$&$\bf0.0$\cr

&$\gamma=3.0$&$69.8$&$46.2$&$31.3$&$23.6$&$16.9$&$12.5$&$11.2$&$8.6$&$7.0$&$2.8$&$1.4$&$\bf0.0$&$0.1$\cr

&$\gamma=4.0$&$80.9$&$53.1$&$38.7$&$30.4$&$21.1$&$17.9$&$14.5$&$11.9$&$8.4$&$4.2$&$2.1$&$0.9$&$0.2$\cr

&$\gamma=5.0$&$90.6$&$66.3$&$43.7$&$36.3$&$25.8$&$24.5$&$20.3$&$16.6$&$12.1$&$5.7$&$3.3$&$1.6$&$0.4$\cr

&$\gamma=6.0$&$103.3$&$73.3$&$49.1$&$42.3$&$34.6$&$28.1$&$25.5$&$19.7$&$14.4$&$8.3$&$4.0$&$2.4$&$1.5$\cr

&$\gamma=7.0$&$110.3$&$79.4$&$55.9$&$45.6$&$38.0$&$31.6$&$29.1$&$24.6$&$19.9$&$10.3$&$5.1$&$3.6$&$2.1$\cr

&$\gamma=8.0$&$118.4$&$86.1$&$63.6$&$50.2$&$41.3$&$35.7$&$33.5$&$27.6$&$21.4$&$12.6$&$5.8$&$4.3$&$2.5$\cr

&$\gamma=9.0$&$124.0$&$92.8$&$68.7$&$54.1$&$46.5$&$39.4$&$36.3$&$30.2$&$25.0$&$13.9$&$7.2$&$4.7$&$3.2$\cr

&$\gamma=10.0$&$129.5$&$98.6$&$74.3$&$57.8$&$49.6$&$42.2$&$39.3$&$32.6$&$27.1$&$17.5$&$7.7$&$5.9$&$4.0$\\\hline

\hline


\end{tabular}
}
\end{center}
\end{table}

\newpage\clearpage

\noindent\textbf{Experiment on diagonal augmentation parameter.} \ We set the diagonal augmentation parameter $\lambda = \gamma \frac{\log(p)}{n}$.
The main result given in the Table~\ref{tab:baseline_results} and Figure~\ref{fig:compare} is tested with $\gamma = 1.0$.
We give the results in Table~\ref{tab:lambda_vp_gauss} on different choices of $\gamma$.

\subsection{Experiments Details with Unobserved Variables Settings}\label{sec:lat_exp_det}

\noindent\textbf{Data Generation.}\
Given the graph topology, we generate the dataset by linear SEM with equal variance Gaussian noise. The noise variance is set to $1.0$, and the edge weight is set to $1.0$.

\vspace{0.1in}
\noindent\textbf{Hyper Parameters.}\
We adopted Adam optimizer with learning rate initialized as 0.05, and decades every 100 gradient step with exponential scheduler with rate 0.99. The sparse loss coefficient $\mu$ is 0.05. The trainable parameters $\mS$ are initialized as $0.5$. For each latent variable, we take the result when the covariance loss less than 0.005 or at most 10 thousand gradient steps. The rounding threshold is set to $0.4$.
The results tested on different settings of hyperparameters (see Table~\ref{tab:latent_parameter} for details) are given in Table~\ref{tab:result_hp}. We can see from the table that the performance is not quite sensitive to the hyperparameters for most of the graph types.


\begin{table*}[h!]

\vspace{-0.1in}

\begin{center}
\caption{Parameter suites \label{tab:latent_parameter}}\vspace{0.1in}
\setlength{\tabcolsep}{0.5mm}{
\begin{tabular}{c|c|c|c|c|c|c}
\hline\hline

\bf Parameter Suite & \bf Optimizer & \bf Learning Rate & \bf Scheduler &  $\mu$ & $\zeta$ & \bf Sample (K)  \cr\hline
A & Adam & 0.05 & 0.99/100 & 0.05 & 0.1 & 5 \cr
B & Adam & 0.05 & 0.99/100 & 0.1 & 0.1 & 5 \cr
C & Adam & 0.05 & 0.99/100 & 0.05 & 0.05 & 5 \cr
D & Adam & 0.05 & 0.90/100 & 0.05 & 0.05  & 5\cr
E & SGD  & 0.005 & 0.99/100 & 0.01 & 0.05 & 5 \cr
F & Adam & 0.05 & 0.99/100 & 0.05 & 0.1  & 10\cr
G & Adam & 0.05 & 0.99/100 & 0.1 & 0.1  & 10\cr
H & Adam & 0.05 & 0.99/100 & 0.05 & 0.05  & 10\cr
I & Adam & 0.05 & 0.90/100 & 0.05 & 0.05  & 10\cr
J & SGD  & 0.005 & 0.99/100 & 0.01 & 0.05 & 10 \cr
K & Adam  & 0.05 & 0.99/100 & 0.1 & 0.1 & 50 \cr
\hline\hline
\end{tabular}
}
\end{center}\vspace{-0.1in}
\end{table*}

\begin{table*}[h!]

\vspace{-0.2in}

\begin{center}
\caption{SHD results on latent variables. Tested on different parameter suites. \label{tab:result_hp}}
\setlength{\tabcolsep}{0.5mm}{
\begin{tabular}{c|c|c|c|c|c|c|c|c|c|c|c|c|c|c|c|c|c|c|c|c|c|c}
\hline\hline
\multirow{2}{*}{\makecell{\bf{Parameter} \\ \bf Suite}}
& \multicolumn{22}{c}{\bf Graph Type} \cr
\cline{2-23}
& 1 & 2 & 3  & 4 & 5 & 6 & 7 & 8 & 9 & 10 & 11&12&13&14&15&16&17&18&19&20&21&22 \\  \hline


A&$\bf0.0$&$\bf0.0$&$\bf0.0$&$\bf0.0$&$\bf0.0$&$\bf0.0$&$\bf0.0$&$\bf0.0$&$\bf0.0$&$\bf0.0$&$\bf0.0$&$\bf0.0$&$\bf0.0$&$\bf0.0$&$\bf0.0$&$\bf0.0$&$10.0$&$\bf0.0$&$4.8$&$\bf0.0$&$1.3$&$1.8$\cr

B&$\bf0.0$&$\bf0.0$&$\bf0.0$&$\bf0.0$&$\bf0.0$&$\bf0.0$&$\bf0.0$&$\bf0.0$&$\bf0.0$&$\bf0.0$&$\bf0.0$&$\bf0.0$&$4.8$&$\bf0.0$&$\bf0.0$&$\bf0.0$&$9.5$&$\bf0.0$&$4.8$&$\bf0.0$&$0.9$&$3.7$\cr

C&$\bf0.0$&$\bf0.0$&$\bf0.0$&$\bf0.0$&$\bf0.0$&$\bf0.0$&$\bf0.0$&$\bf0.0$&$\bf0.0$&$\bf0.0$&$\bf0.0$&$\bf0.0$&$1.1$&$\bf0.0$&$\bf0.0$&$\bf0.0$&$3.4$&$\bf0.0$&$4.0$&$\bf0.0$&$1.9$&$1.9$\cr

D&$\bf0.0$&$\bf0.0$&$\bf0.0$&$\bf0.0$&$\bf0.0$&$\bf0.0$&$\bf0.0$&$\bf0.0$&$\bf0.0$&$\bf0.0$&$\bf0.0$&$\bf0.0$&$\bf0.0$&$\bf0.0$&$\bf0.0$&$\bf0.0$&$2.9$&$\bf0.0$&$3.9$&$\bf0.0$&$2.1$&$2.2$\cr

E&$\bf0.0$&$\bf0.0$&$\bf0.0$&$\bf0.0$&$\bf0.0$&$\bf0.0$&$\bf0.0$&$\bf0.0$&$\bf0.0$&$\bf0.0$&$\bf0.0$&$\bf0.0$&$3.5$&$\bf0.0$&$\bf0.0$&$2.0$&$10.0$&$\bf0.0$&$3.0$&$2.0$&$9.8$&$7.0$\cr

F&$\bf0.0$&$\bf0.0$&$\bf0.0$&$\bf0.0$&$\bf0.0$&$\bf0.0$&$\bf0.0$&$\bf0.0$&$\bf0.0$&$\bf0.0$&$\bf0.0$&$\bf0.0$&$\bf0.0$&$\bf0.0$&$\bf0.0$&$\bf0.0$&$10.0$&$\bf0.0$&$5.0$&$\bf0.0$&$1.2$&$1.9$\cr

G&$\bf0.0$&$\bf0.0$&$\bf0.0$&$\bf0.0$&$\bf0.0$&$\bf0.0$&$\bf0.0$&$\bf0.0$&$\bf0.0$&$\bf0.0$&$\bf0.0$&$\bf0.0$&$0.9$&$\bf0.0$&$\bf0.0$&$\bf0.0$&$10.0$&$\bf0.0$&$4.7$&$\bf0.0$&$0.9$&$3.4$\cr

H&$\bf0.0$&$\bf0.0$&$\bf0.0$&$\bf0.0$&$\bf0.0$&$\bf0.0$&$\bf0.0$&$\bf0.0$&$\bf0.0$&$\bf0.0$&$\bf0.0$&$\bf0.0$&$\bf0.0$&$\bf0.0$&$\bf0.0$&$\bf0.0$&$3.0$&$\bf0.0$&$3.9$&$\bf0.0$&$1.3$&$\bf1.3$\cr

I&$\bf0.0$&$\bf0.0$&$\bf0.0$&$\bf0.0$&$\bf0.0$&$\bf0.0$&$\bf0.0$&$\bf0.0$&$\bf0.0$&$\bf0.0$&$\bf0.0$&$\bf0.0$&$\bf0.0$&$\bf0.0$&$\bf0.0$&$\bf0.0$&$\bf2.6$&$\bf0.0$&$4.0$&$\bf0.0$&$1.3$&$2.8$\cr

J&$\bf0.0$&$\bf0.0$&$\bf0.0$&$\bf0.0$&$\bf0.0$&$\bf0.0$&$\bf0.0$&$\bf0.0$&$\bf0.0$&$\bf0.0$&$\bf0.0$&$\bf0.0$&$5.2$&$\bf0.0$&$\bf0.0$&$2.0$&$8.0$&$\bf0.0$&$\bf2.0$&$1.9$&$3.7$&$7.5$\cr

K&$\bf0.0$&$\bf0.0$&$\bf0.0$&$\bf0.0$&$\bf0.0$&$\bf0.0$&$\bf0.0$&$\bf0.0$&$\bf0.0$&$\bf0.0$&$\bf0.0$&$\bf0.0$&$\bf0.0$&$\bf0.0$&$\bf0.0$&$\bf0.0$&$10.0$&$\bf0.0$&$3.8$&$\bf0.0$&$\bf0.0$&$1.7$\cr

 \hline\hline
\end{tabular}
}
\end{center}
\end{table*}

\newpage

\subsection{Baseline Implementations}
The baselines are implemented via the codes provided from the following links:
\begin{itemize}
\item NOTEARS, NOTEARS-MLP: \href{https://github.com/xunzheng/notears}{https://github.com/xunzheng/notears}
\item NPVAR: \href{https://github.com/MingGao97/NPVAR}{https://github.com/MingGao97/NPVAR}
\item EQVAR, LISTEN: \href{https://github.com/WY-Chen/EqVarDAG}{ https://github.com/WY-Chen/EqVarDAG}
\item CORL: \href{https://github.com/huawei-noah/trustworthyAI/tree/master/gcastle}{https://github.com/huawei-noah/trustworthyAI/tree/master/gcastle}
\item DAG-GNN: \href{https://github.com/fishmoon1234/DAG-GNN}{https://github.com/fishmoon1234/DAG-GNN}
\item B-S \href{https://proceedings.neurips.cc/paper/2021/file/c0f6fb5d3a389de216345e490469145e-Supplemental.zip}{https://proceedings.neurips.cc/paper/2021/file/c0f6fb5d3a389de216345e490469145e-Supplemental.zip}
\end{itemize}


\end{document}

%% file: math_commands.tex
\usepackage{amsmath,amsfonts,bm,amssymb,mathtools,amsthm}

\newcommand{\parents}{Pa}

\DeclareMathOperator{\diag}{diag}

\DeclareMathOperator{\argmin}{argmin}

\newcommand{\T}{{\rm T}}
\newcommand{\wht}{\widehat}
\newcommand{\wtd}{\widetilde}
\newcommand{\bsmat}{\left[\begin{smallmatrix}}
\newcommand{\esmat}{\end{smallmatrix}\right]}

\newtheorem{theorem}{Theorem}[section]

\newtheorem{lemma}[theorem]{Lemma}
\newtheorem{proposition}{Proposition}
\newtheorem{remark}{Remark}

\newcommand{\E}{\mathbb{E}}
\newcommand{\R}{\mathbb{R}}

\def\rvw{{\mathbf{w}}}
\def\rvx{{\mathbf{x}}}

\def\vi{{\bm{i}}}
\def\vj{{\bm{j}}}

\def\vn{{\bm{n}}}

\def\vu{{\bm{u}}}
\def\vv{{\bm{v}}}
\def\vw{{\bm{w}}}
\def\vx{{\bm{x}}}
\def\vy{{\bm{y}}}
\def\vz{{\bm{z}}}

\def\mA{{\bm{A}}}

\def\mC{{\bm{C}}}

\def\mE{{\bm{E}}}
\def\mF{{\bm{F}}}
\def\mG{{\bm{G}}}

\def\mI{{\bm{I}}}

\def\mL{{\bm{L}}}

\def\mN{{\bm{N}}}

\def\mP{{\bm{P}}}

\def\mS{{\bm{S}}}
\def\mT{{\bm{T}}}
\def\mU{{\bm{U}}}
\def\mV{{\bm{V}}}
\def\mW{{\bm{W}}}
\def\mX{{\bm{X}}}

\def\mSigma{{\bm{\Sigma}}}

\DeclareMathAlphabet{\mathsfit}{\encodingdefault}{\sfdefault}{m}{sl}
\SetMathAlphabet{\mathsfit}{bold}{\encodingdefault}{\sfdefault}{bx}{n}

\def\gA{{\mathcal{A}}}

\def\gG{{\mathcal{G}}}

\def\gI{{\mathcal{I}}}

\def\gL{{\mathcal{L}}}
\def\gM{{\mathcal{M}}}

\def\gO{{\mathcal{O}}}

\def\emN{{N}}

\def\emX{{X}}
